%% file: main.tex
\documentclass[11pt]{article}

\usepackage[utf8]{inputenc}
\usepackage[T1]{fontenc}
\usepackage[english]{babel}

\usepackage{booktabs}
\usepackage{adjustbox}
\usepackage{mathtools}
\usepackage{tabularx}
\usepackage{array}
\usepackage[shortlabels]{enumitem}
\usepackage{amsmath} \DeclareRobustCommand{\widebar}[1]{\overline{#1}}
\usepackage{amssymb}
\usepackage{caption}
\usepackage[right=1in, top=1in, bottom=1in, left=1in]{geometry}
\usepackage{lmodern}
\usepackage{setspace}
\usepackage[dvipsnames]{xcolor}
\usepackage[colorlinks=true, linkcolor=blue, citecolor=blue, hypertexnames=false]{hyperref}
\usepackage{nameref}
\usepackage{cleveref}
\usepackage[authoryear,round]{natbib}
\usepackage{amsthm}

\DeclarePairedDelimiter{\ind}{\mathbf{1}\{}{\}}

\usepackage{algorithm}
\usepackage{algpseudocode}
\usepackage{float}
\floatname{algorithm}{ALGORITHM}
\algrenewcommand\alglinenumber[1]{\footnotesize #1:}
\algrenewcommand\algorithmiccomment[1]{\hfill{\(\triangleright\)~#1}}

\renewcommand{\arraystretch}{1.15}

\newtheorem{lemma}{Lemma}

\newtheorem{corollary}{Corollary}
\newtheorem{theorem}{Theorem}
\newtheorem*{theorem*}{Theorem}
\newtheorem{assumption}{Assumption}

\newtheorem{definition}{Definition}

\newtheorem{remark}{Remark}

\usepackage[runin]{abstract}

\abslabeldelim{.\quad}

\input{macros_ec}

\commenter{sohom}{cyan}
\commenter{duy}{green}
\commenter{yunbei}{red}
\commenter{richard}{magenta}

\title{In-Context Learning for Data-Driven Censored Inventory Control}

\author{\begingroup
\setlength{\tabcolsep}{6pt}
\renewcommand{\arraystretch}{0.88}
\begin{tabular}{cc}
\begin{tabular}{@{}c@{}}
Sohom Mukherjee\thanks{Equal contribution.}\\[-0.35ex]
{\footnotesize Julius-Maximilians-Universit\"at W\"urzburg}\\[-0.55ex]
{\footnotesize \texttt{sohom.mukherjee@uni-wuerzburg.de}}
\end{tabular}
&
\begin{tabular}{@{}c@{}}
Anh-Duy Pham\footnotemark[1]\\[-0.35ex]
{\footnotesize Julius-Maximilians-Universit\"at W\"urzburg}\\[-0.55ex]
{\footnotesize \texttt{anh-duy.pham@uni-wuerzburg.de}}
\end{tabular}
\\[5.0ex]
\begin{tabular}{@{}c@{}}
Richard Pibernik\\[-0.35ex]
{\footnotesize Julius-Maximilians-Universit\"at W\"urzburg}\\[-0.55ex]
{\footnotesize Zaragoza Logistics Center}\\[-0.55ex]
{\footnotesize \texttt{richard.pibernik@uni-wuerzburg.de}}
\end{tabular}
&
\begin{tabular}{@{}c@{}}
Yunbei Xu\\[-0.35ex]
{\footnotesize National University of Singapore}\\[-0.55ex]
{\footnotesize \texttt{yunbei@nus.edu.sg}}
\end{tabular}
\end{tabular}
\endgroup
}
\date{}

\begin{document}

\maketitle

\begin{center}
\textbf{Abstract}
\end{center}
\vspace{-0.6em}

\begingroup
\setstretch{1.0}
\footnotesize
\noindent
We study inventory control with decision-dependent censoring. While this class encompasses a wide range of operational decision problems with continuous action spaces, we focus on the canonical example of the censored or repeated newsvendor (R-NV). Each period, the decision-maker chooses an order quantity and observes the sales, so the demand is only partially observed. Two complementary schools of work have been popular in this domain, but suffer from their own limitations: parametric Thompson sampling (might fail under prior mismatch), and offline imputation (might fail to transfer to online environments). Based on the recent line of work on the predictive view of decision making, we take a fresh look at combining the two approaches above, by taking oracle actions on learned completions of the latent demand. We propose in-context generative posterior sampling (ICGPS) that makes this possible via modern generative models capable of offline meta-training and online in-context autoregressive generation. Theoretically, we show that the Bayesian regret of the deployed ICGPS policy using the learned completion kernel is bounded by the Bayesian regret of the TS benchmark with an ideal completion kernel plus a deployment penalty that scales as $\sqrt{T}$ times the square root of the completion mismatch. This provides a template where one can plug in the Bayesian regret of TS for any known operational problem. In particular, for the R-NV problem, we derive a sublinear Bayesian regret by showing that censored feedback can be reduced to a bandit convex optimization feedback. Moreover, the completion mismatch is controlled by the offline predictive mismatch, i.e., offline quality translates to online performance, under reasonable assumptions. We instantiate ICGPS in practice by proposing a novel ChronosFlow architecture that combines a frozen time-series transformer backbone with a trainable conditional normalizing-flow head for fast conditional sampling. ChronosFlow-ICGPS performs at par with TS and outperforms Myopic and UCB-style baselines for benchmark experiments, and shows robustness against prior mismatch and distribution shift. ChronosFlow-ICGPS also performs well for the real-world SuperStore dataset, especially under heavy censoring.

\endgroup

\input{sec1-intro}

\input{sec2-general_framework}

\input{sec3-repeated_nv}
\input{sec4-experiments}

{\setstretch{1.0}
\bibliographystyle{plainnat}
\bibliography{ec-bibliography}
}
\newpage
\appendix
\counterwithin{table}{section}
\renewcommand{\thetable}{\thesection.\arabic{table}}

\input{sec-apx_method}

\input{sec-apx_theory}

\input{sec-apx_experiment}

\end{document}

%% file: sec1-intro.tex
\section{Introduction}
\label{sec:intro}

Many operational decision problems involve decision-dependent uncertainty: the decisions we take determine what we get to observe. A canonical example is retail inventory under stockouts, where we observe sales but not demand whenever inventory is insufficient. This \emph{right-censoring} creates an exploration--exploitation tension: ordering conservatively reduces immediate holding costs but also increases censoring, making future learning harder. The aforementioned problem, usually referred to as the repeated or censored newsvendor, has therefore become a central testbed for understanding learning under censored feedback.
While Bayesian methods, dating back to the popular posterior sampling approach introduced by Thompson \cite{thompson1933likelihood}, have proven to be powerful in solving this problem \cite{zhang2025thompson}, the bottleneck of using such methods in modern practice remains the choice of the prior distribution \cite{xu2025bayesian}. On the other hand, imputation methods that perform some form of demand de-censoring are confined mainly to offline settings and fail to provide online guarantees \cite{clausen2025iterative}.

To overcome this challenge, we adopt a missing-data view of uncertainty in sequential decision-making.
Rather than viewing learning as estimating parameters, we treat the unobserved components of the trajectory as \emph{missing outcomes} that can be \emph{completed} in a way that is consistent with the feedback mechanism.
If one could sample a completion from the true posterior, then posterior sampling would act by drawing a plausible complete
trajectory and choosing the corresponding oracle action. Recent work has shown that this perspective can be made algorithmic by replacing explicit posteriors with modern generative models that can sample missing data \emph{in-context}, i.e., only by conditioning on examples provided in the context, and no re-training \citep{cai2024active, zhang2025contextual}. However, such work has been confined to finite discrete action spaces and uncensored feedback, limiting their application to operational problems with decision-dependent uncertainty.

In this paper, we generalize this work to solve a broad class of operational sequential decision problems where feedback is decision-dependent (censoring) and the decision (action) space is continuous and scalar. 
Our goal is online learning with performance guarantees in continuous action spaces: we seek sublinear regret against the
Bayes-optimal benchmark induced by the data-generating model. We propose in-context \emph{Generative Posterior Sampling (GPS)} for the aforementioned class of operational problems. Offline, we learn a conditional generative model that can \emph{complete} the latent trajectory given the observed
history (including censoring indicators) by autoregressive sampling.
Online, at each time, GPS draws a completion and takes the oracle decision. In particular, for the repeated newsvendor, the oracle is the critical-fractile map, and the completion model generates demand trajectories consistent with right-censored sales data.

To make in-context GPS (ICGPS) practical at scale, we instantiate the completion model using a modular architecture, \emph{ChronosFlow}. ChronosFlow combines a pretrained probabilistic time-series backbone (Chronos-2) \citep{ansari2025chronos2}, which is a successor of Chronos \citep{ansari2024chronos}, with a lightweight
conditional flow head that supports fast conditional sampling and exact enforcement of censoring constraints. Crucially, ChronosFlow is trained \emph{offline} and then \emph{frozen} at deployment: online learning is performed entirely through in-context conditioning on the growing history, avoiding expensive online gradient updates while still producing posterior-sampling-like exploration.

\subsection{Contributions}

\emph{Methodology.} Our first contribution is a methodological one whereby we extend the ICGPS framework proposed in \cite{cai2024active, zhang2025contextual} to continuous action spaces and censored feedback. Continuous actions are handled in theory by proving an equivalence between function and outcome posterior sampling, and in practice using a generative architecture with a conditional flow head for sampling continuous distributions. Decision-dependent censoring makes the generative step nontrivial, since the completions are not arbitrary and must satisfy hard \emph{consistency constraints} induced by the censoring map. This necessitates using a censored negative log likelihood loss in the offline training, and constraint-aware sampling in the online phase. In particular, we derive the censoring-consistent autoregressive factorization and a tail-conditioning procedure that guarantees generated completions satisfy the censoring constraints, leading to the online Algorithm~\ref{alg:nv-gps}.

\emph{Theory.} We provide a regret decomposition that separates the regret of the deployed GPS policy (using a learned completion kernel) into the regret of the \emph{ideal} GPS policy (using the true kernel) plus a \emph{deployment penalty} controlled by an on-policy completion mismatch. We show that under mild conditions, the R-NV problem admits reduction to a feedback structure identical to that of posterior sampling for a convex bandit, and can therefore inherit $\widetilde O(\sqrt{T})$ Bayesian regret for the first term of the above decomposition. This provides the first Bayesian regret analysis of TS for the R-NV problem, using the information-ratio framework, and might be of independent interest. For the second term, we derive concrete sufficient conditions\footnote{Censoring creates fundamental \emph{identifiability} and \emph{coverage} issues: the learner may never see certain regions of the latent space under conservative policies, so offline predictive quality may fail to translate to online decision quality, without additional assumptions \citep{hssaine2024censored}.} under which censoring-aware offline training objectives control the completion mismatch.

\emph{Experiments.} We propose ChronosFlow, a practical instantiation of ICGPS that produces censoring-consistent completions via a conditional flow head and exact tail-conditioning. Experiment~\ref{subsec:exp1} is a correctly specified sanity check (Weibull), where ChronosFlow-ICGPS matches the conjugate TS baseline across service levels and censoring severities. Experiment~\ref{subsec:exp2} isolates the offline$\to$online link: holding the online GPS wrapper fixed, we vary (i) offline data scale and (ii) CNF head capacity, and observe a monotone relationship between censoring-aware validation fit ($\widehat{\Delta}_{\mathrm{obs}}$) and online regret, consistent with the deployment-penalty interpretation. Experiment~\ref{subsec:exp3} then stresses robustness under (a) out-of-family distribution shift and (b) severe prior mismatch, where ChronosFlow-ICGPS is substantially less brittle than parametric TS variants. Finally, Experiment~\ref{subsec:exp4} benchmarks real censored-demand datasets; the Meta-trained variant yields the largest gains in the most heavily censored regimes while remaining competitive when censoring is mild.

\begin{figure}[t]
    \centering
    \includegraphics[width=0.75\linewidth]{"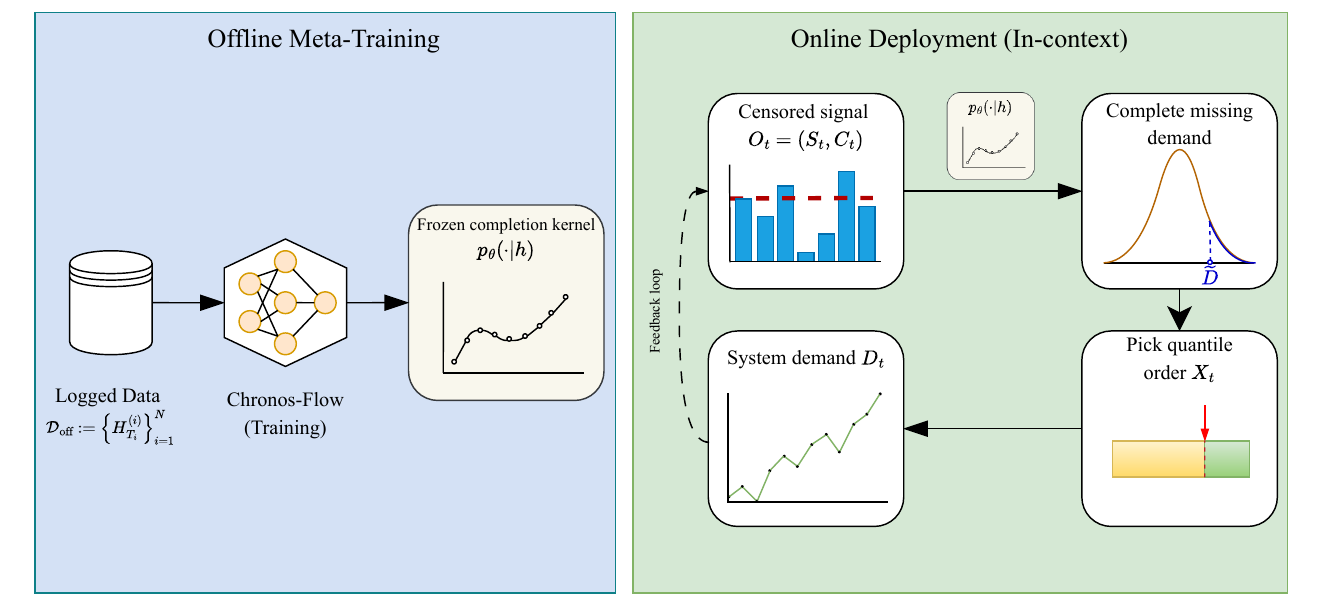"}
    \caption{Schematic diagram for our proposed in-context GPS applied to the repeated newsvendor problem.}
    \label{fig:concept}
\end{figure}

\subsection{Related Work}
\label{subsec:rw}

\emph{Posterior sampling and Bayesian regret.}
Thompson sampling (TS), the most common form of posterior sampling,  dates back to \citet{thompson1933likelihood} and has become a central algorithmic paradigm for exploration in Bayesian bandits.
A large body of theoretical literature establishes sublinear regret guarantees in increasingly rich stochastic models, beginning with
logarithmic regret for classical multi-armed bandits \citep{agrawal2012ts}, and extending to general Bayesian regret
bounds via information-theoretic and posterior-variance arguments \citep{russo2016info,russo2017tutorial}.
Our work shares the same Bayesian regret objective, but the uncertainty we address is induced by \emph{censored feedback}. We implement posterior sampling by \emph{generating missing outcomes} consistent with the censoring mechanism, and our regret bounds separate an ideal Thompson-sampling term from a learned-model \emph{deployment penalty}.

\emph{Bandit convex optimization and continuous-action bandits.}
Bandit convex optimization (BCO) studies online optimization of an unknown convex loss function using only bandit
evaluations.
Foundational work introduced gradient-estimation approaches for bandit convex optimization \citep{flaxman2005bandit_oco},
and continuum-armed bandits more broadly \citep{kleinberg2004continuum}.
In one dimension, the minimax regret improves to $\widetilde{\Theta}(\sqrt{T})$ under convexity and boundedness
\citep{bubeck2015bco1d}, and recent work proves that Thompson sampling itself attains $\widetilde O(\sqrt{T})$
Bayesian regret for BCO with scalar actions \citep{bakhtiari2025bco}. We leverage this result to obtain the Bayesian regret bound for TS on the R-NV problem, via a reduction of the censored feedback to the BCO feedback.

\emph{Posterior sampling via generation.}
A recent line of work reinterprets posterior sampling as \emph{sampling missing data} and then optimizing an oracle
decision rule on the completed dataset, enabling posterior sampling to be implemented by a generative model rather than an explicit posterior \citep{cai2024active,zhang2025contextual}.
This aligns with the broader \emph{prediction-centric} view of uncertainty quantification \citep{fortini2023prediction, Shen2024PredictionCentricUQ,shirvaikar2024general} and decision-making \citep{wen2021predictions}.
Our framework builds directly on the missing-data perspective, but extends it in two directions needed for operational
decision-making: (i) we connect completion sampling to continuous-action convex regret guarantees via BCO, and (ii) we
develop censoring-consistent sampling online, and censored-likelihood training objectives offline, which are tailored to the R-NV problem.

\emph{Learning the newsvendor under censored demand.}
Learning under stockouts has a long history in operations and revenue management.
Early distribution-free approaches for censored newsvendor-type problems include \citet{godfrey2001adaptive},
while \citet{huh2009nonparametric} provides a nonparametric asymptotic analysis under censored demand.
A prominent estimator-based approach uses Kaplan--Meier methods to correct censoring and yields adaptive inventory control
policies with performance guarantees \citep{kaplan1958nonparametric,huh2011adaptive}.
Recent work sharpens the exploration--exploitation perspective and quantifies when active exploration is essential
\citep{besbes2022tradeoff}, while data-driven formulations highlight identifiability barriers under severe censoring
\citep{hssaine2024censored}.
In contrast, we treat the censored newsvendor as an instantiation of a general missing-data posterior-sampling framework:
our algorithm explores by sampling \emph{feasible demand completions}.

\emph{Data-driven approaches for censored newsvendor.}
A growing literature develops offline procedures to correct censoring and to improve order decisions in the R-NV using robust or data-driven approaches.
Examples include data-driven DRO corrections for censored demand \citep{su2025bridging}, iterative maximum-likelihood
procedures tailored to censored observations \citep{clausen2025iterative}, and target-oriented data-driven policies
\citep{wang2025target}.
These works focus primarily on \emph{offline} decision quality and robustness, whereas our focus is \emph{sequential}
decision-making: we integrate censoring-consistent modeling and sampling into a posterior-sampling-style online policy and
provide regret guarantees that explicitly quantify the cost of using a learned completion model.

\emph{Meta-learning in bandits.}
Many operational settings exhibit repeated learning across related tasks (e.g., products, regions, seasons), motivating
transfer and meta-learning approaches.
In revenue management, \citet{bastani2022metapricing} studies meta dynamic pricing and shows how learning a shared prior
across experiments can reduce regret; more broadly, meta-learning with bandit feedback has been studied in adversarial and
stochastic formulations \citep{kveton2021meta, khodak2023metalearningbandits}.
Our empirical design similarly evaluates offline pretraining across heterogeneous demand distributions, but our algorithmic
mechanism is different: we transfer via an offline-trained \emph{completion model} that is used online purely through
in-context conditioning.

%% file: sec2-general_framework.tex
\section{Background}
\label{sec:framework}

This section introduces the repeated newsvendor (R-NV) problem along with other problems belonging to the same class of decision-making problems. This is followed by a background on generative posterior sampling (GPS) and in-context GPS (ICGPS) via autoregressive completion. The section ends with a generic analysis of ICGPS for the R-NV problem.

\subsection{Inventory control with censored demand}
\label{sec:nv_running_example}

We begin by describing the problem setup for inventory control with right-censored demand (repeated newsvendor), which will serve as a running example of an operational sequential decision-making problem with continuous actions and decision-dependent uncertainty. A general template for the aforementioned class of operational decision problems, that can be solved using our proposed method, is provided in Appendix~\ref{app:general_template}. In particular, Table~\ref{tab:app_examples} lists additional instantiations of the same class, including booking-limit (capacity controls) in revenue management \cite{TalluriVanRyzin2004,littlewood2005forecasting}, budget pacing under spend caps in digital advertising \cite{BalseiroEtAl2023RobustPacing,ConitzerEtAl2022PacingEq}, and posted-price mechanisms \cite{KleinbergLeighton2003,Myerson1981}.

\paragraph{Interaction and censoring.}
Fix a horizon $T$, and let the \emph{action} space be $\mathcal{X}:=[0,B]$ and the \emph{demand} space be $\mathcal{D}:=[0,B]$. At each period $t=1,\dots,T$, the decision-maker chooses an order quantity $X_t \in \mathcal X$. The environment then realizes a latent demand $D_t \in \mathcal D$, whose random trajectory is given by $D_{1:T} := (D_1,\dots,D_T)$. Under lost sales, the decision-maker observes only the right-censored feedback, which is, in this case, the sales and a stockout indicator
\[
S_t := \min\{D_t, X_t\},
\qquad
C_t := \ind{D_t \le X_t},
\qquad
O_t := (S_t, C_t) =: \psi(X_t, D_t).
\]
The history at time $t$ is \( H_t := \bigl( (X_s, O_s) \bigr)_{s=1}^t \in (X \times \mathcal{O})^t, H_0 := \emptyset, \)
and a (possibly randomized) policy $\pi$ maps $H_{t-1}$ to a distribution over $X$, from which $X_t$ is drawn. The per round loss is a known measurable function, given by the classical newsvendor cost $\ell(x,d)=h(x-d)_+ + b(d-x)_+$, where $h > 0$ is the unit overage cost, $b > 0$ is the unit underage cost, and we denote $(a)_+ := \max\{a, 0\}$.

\paragraph{Bayesian environment.}
The unknown environment is captured by a distribution $\P^\star$ over complete trajectories $D_{1:T}$.
All expectations below are with respect to $\P^\star$ and the policy randomness, unless stated otherwise. Let $x^\star \in \arg\min_{x\in\cX} f^\star(x)$ be an optimal action, and define the \emph{Bayes risk function} \( f^\star(x) := \E_{\P^\star}\bigl[\ell(x,D_t)\bigr]\). Let us denote the per-round regret as $r_t := f^\star(X_t) - f^\star(x^\star)$. The performance criterion is given by the \emph{cumulative} Bayesian regret, which is measured as
\[
\BReg_T(\pi;\P^\star) := \E_{\P^\star_\pi}\Bigl[\sum_{t=1}^T \bigl(f^\star(X_t)-f^\star(x^\star)\bigr)\Bigr].
\]

\paragraph{Oracle map}
We denote the critical fractile by $\gamma:=\frac{b}{b+h}\in(0,1)$. Let $d_{1:T}=(d_1,\dots,d_T)\in [0, B]^T$ be a completed demand trajectory. Define the empirical risk $\widehat f_{d_{1:T}}(x)
:= \frac1T \sum_{t=1}^T \ell(x,d_t)$. We define the oracle action $\hat x(\cdot)$ by empirical risk minimizer: $\hat x(d_{1:T}) \in \arg\min_{x\in[0,B]}\ \widehat f_{d_{1:T}}(x)$. Define the empirical CDF $\widehat F_T(x) := \frac1T \sum_{t=1}^T \ind{ d_t \le x}$, and its left limit $\widehat F_T(x^-):=\lim_{y\uparrow x}\widehat F_T(y)
=\frac1T\sum_{t=1}^T \ind{ d_t < x}$. In particular, for the repeated newsvendor problem, the oracle equals the \emph{left empirical $\gamma$-quantile} $\hat x(d_{1:T}) = \inf\{x: \widehat F_T(x)\ge \gamma\}$.

\subsection{Generative posterior sampling}
\label{subsec:missingdata}

Generative posterior sampling (GPS) is based on the missing data view of uncertainty put forward by the literature on predictive Bayes \citep{Shen2024PredictionCentricUQ, shirvaikar2024general,cai2024active,zhang2025contextual}. The key intuition is that \emph{uncertainty arises because we have not yet observed a complete dataset of outcomes.} If we had the complete dataset, we could fit an oracle decision rule and act without uncertainty. In our setting, the \emph{complete dataset} is the full trajectory $D_{1:T}$. 

A natural oracle action given $D_{1:T}$, for the purpose of illustration, is the empirical risk minimizer $\hat x(D_{1:T}) \in \arg\min_{x\in\cX}\; \hat f_{D_{1:T}}(x)$, where $\hat f_{D_{1:T}}(x) := \frac1T \sum_{t=1}^T \ell(x,D_t)$. At decision time $t$, the learner has only a partial history $H_{t-1}$. Under the Bayesian model $\P^\star$, this induces a posterior distribution over the \emph{complete} trajectory:
\(
\P^\star(\,\cdot\mid H_{t-1}).
\)
A generative posterior-sampling strategy (Algorithm \ref{alg:gps}) draws a plausible completion
$\widetilde D_{1:T}^{(t)} \sim \P^\star(\cdot\mid H_{t-1})$ and acts as if it were true by selecting
\begin{equation}
\label{eq:ideal-ps}
X_t = \hat x(\widetilde D_{1:T}^{(t)}).
\end{equation}

\begin{lemma}[Probability matching]\label{lem:probmatching}
Conditioned on $H_{t-1}$, the action $X_t$ in \eqref{eq:ideal-ps} is distributed as
the posterior distribution of the oracle action $\hat x(D_{1:T})$, i.e., we have:
\(
\Law(X_t\mid H_{t-1}) = \Law\bigl(\hat x(D_{1:T})\mid H_{t-1}\bigr).
\)
\end{lemma}

\begin{proof}
Since $\widetilde D_{1:T}^{(t)}\mid H_{t-1} \sim \P^\star(\cdot\mid H_{t-1})$ by construction, and
$X_t$ is a measurable function of $\widetilde D_{1:T}^{(t)}$, the claim follows by the
pushforward measure identity: for any measurable $A\subseteq\cX$,
\(
\P(X_t\in A\mid H_{t-1})
=
\P\bigl(\hat x(\widetilde D_{1:T}^{(t)})\in A \mid H_{t-1}\bigr)
=
\P\bigl(\hat x(D_{1:T})\in A \mid H_{t-1}\bigr).
\)
\end{proof}

\subsection{In-context GPS via autoregressive completion}
\label{subsec:gen-ps}

In most operational problems, the true environment law $\P^\star$ is unknown, hence we cannot sample
$\P^\star(\cdot\mid H_{t-1})$ exactly. Following \citet{cai2024active,zhang2025contextual}, we instead
learn a \emph{generative simulator} offline and deploy it online through \emph{in-context conditioning}:
\begin{itemize}
    \item \textit{A learned simulator of complete outcomes.}
    Fix a model class $\{p_\theta:\theta\in\Theta\}$ of distributions over complete trajectories $D_{1:T}$,
    together with conditionals $p_\theta(\cdot\mid H_{t-1})$ that can be sampled. A standard implementation is an autoregressive factorisation $p_\theta(D_{1:T}) = \prod_{t=1}^T p_\theta(D_t\mid D_{1:t-1})$, trained by next-outcome prediction (with modifications when only coarsened feedback is available;
    see Section~\ref{sec:newsvendor}).
    \item \textit{Generative posterior sampling (GPS).}
    Online, at each time $t$, GPS samples a plausible completion from the learned conditional model and
    then applies the oracle map $\hat x(\cdot)$: $\widetilde D_{1:T}^{(t)} \sim p_\theta(\cdot\mid H_{t-1}), X_t \gets \hat x\!\bigl(\widetilde D_{1:T}^{(t)}\bigr).$ Crucially, \emph{no online parameter updates are required}; all adaptation occurs by conditioning on
    $H_{t-1}$ (in-context learning).
\end{itemize}

\begin{algorithm}[t]
\caption{In-context Generative Posterior Sampling (ICGPS)}
\label{alg:gps}
\small
\begin{algorithmic}[1]
\Require action set $\cX$, oracle map $\hat x(\cdot)$, learned simulator $p_\theta$
\State $H_0 \gets \emptyset$
\For{$t \gets 1,2,\ldots,T$}
  \State Sample completion $\widetilde D_{1:T}^{(t)} \sim p_\theta(\cdot\mid H_{t-1})$
  \Comment{completed trajectory conditional on history}
  \State Choose $X_t \gets \hat x(\widetilde D_{1:T}^{(t)})$
  \Comment{oracle action on the completion trajectory}
  \State Observe $O_t=\psi(X_t,D_t)$ and set $H_t \gets (H_{t-1},(X_t,O_t))$
  \Comment{update history}
\EndFor
\end{algorithmic}
\end{algorithm}

\subsection{Deployment penalty via completion kernel mismatch}\label{subsec:reduction}

In this Section, we quantify how much extra Bayesian regret is incurred when in-context GPS is deployed with a learned completion model instead of the ideal completion kernel. We denote by \(H_{t-1}\) the
history up to time \(t-1\). For any policy $\pi$, let $\P^\star_\pi$ denote the induced law of the interaction history
$H_T=((X_t,O_t))_{t=1}^T$ when $\pi$ is run in the true environment $\P^\star$. We compare the \emph{ideal} policy $\pi_\star := \pi_{\P^\star}$, which uses the \emph{true} completion kernel $\P^\star(\cdot\mid H_{t-1})$, and the \emph{deployed} policy $\pi_\theta := \pi_{p_\theta}$, which uses a learned kernel $p_\theta(\cdot\mid H_{t-1})$. 
At time \(t\), the \emph{true} completion kernel is the conditional law
\(
Q_t^\star(\,\cdot\,|h) := \Law_{\P^\star}\!\bigl(D_{1:T}\in\cdot \,\big|\, H_{t-1}=h\bigr),
\)
and a learned completion model with parameter \(\theta\) induces a kernel on the same space given by \(Q_{t,\theta}(\,\cdot \mid h) := \mathrm{Law}_\theta\bigl(\widetilde D_{1:T}^{(t)} \mid H_{t-1}=h\bigr),\).

\begin{definition}[On-policy completion mismatch]\label{def:completion-mismatch}
The cumulative \emph{completion mismatch} of \(Q_{t,\theta}\) relative to \(Q_t^\star\), measured along the
deployment trajectory of \(\pi_\theta\), is given by the following expression
\[
\Delta_T^{\mathrm{comp}}(\theta;\pi_\theta)
\ :=\
\sum_{t=1}^T
\mathbb{E}_{H_{t-1}\sim \P^\star_{\pi_\theta}}
\Bigl[
\KL\!\bigl(Q_t^\star(\cdot|H_{t-1})\,\|\,Q_{t,\theta}(\cdot|H_{t-1})\bigr)
\Bigr].
\]
\end{definition}

\begin{theorem}[Deployment penalty bound under $\P^\star$]
\label{thm:sec25_deploy}
Let $\pi_\star=\pi_{\P^\star}$ be the ideal GPS policy that uses the true completion kernel, and
let $\pi_\theta=\pi_{p_\theta}$ be the deployed GPS policy using the learned kernel. Assume there exists \(V_{\mathrm{comp}}\in[1,\infty)\) such that for all $t\in[T]$ and all histories \(h\in\mathcal{H}_{t-1}\),
\(Q_{t,\theta}(\cdot\mid h)\ll Q^\star_t(\cdot\mid h)\) and
\(
\sup_{F\in\mathcal{F}_{\mathcal{U}}}
\frac{Q_{t,\theta}(F\mid h)}{Q^\star_t(F\mid h)}
\;\le\;
V_{\mathrm{comp}}
\).
Moreover, assume that the per-round regret $r_t$ satisfies $\E_{\P^\star}[r_t \mid H_{t-1}] = \E_{\P^\star}[r_t]$ $\P^\star$-a.s..
Then
\begin{equation}
\label{eq:sec25_deploy}
\BReg_T(\pi_\theta;\P^\star)
\le
\underbrace{\BReg_T(\pi_\star;\P^\star)}_{\text{Term I: TS Bayes Regret}}
+
\underbrace{B^\prime\sqrt{\frac{T\left(2+\log V_{\mathrm{comp}}\right)}{2}\,\Delta^{\mathrm{comp}}_{T}(\theta;\pi_\theta)}}_{\text{Term II: Deployment Penalty}}
\,,
\end{equation}
where $B^\prime$ denotes the range length of the per round regret $r_t$, i.e, we have $\sup r_t - \inf r_t \le B'$ a.s.
\end{theorem}

\paragraph{Proof Sketch.}
We provide a proof sketch here, with the main technical components, and refer the reader to
Appendix~\ref{apx:first_proof} for the full proof. Let \(P\coloneqq \P^\star_{\pi_\theta}\) and \(Q\coloneqq \P^\star_{\pi_\star}\). We denote the per-round regret as $r_t := f^\star(X_t) - \min_{x \in \mathcal X}f^\star(x)$, and note that $r_t$ has range length $B^\prime$. The cumulative regret is denoted as $R_T := \sum_{t=1}^T r_t$, and the cumulative Bayesian regret for a policy $\pi$ under environment $\P^\star$ is denoted by $\mathrm{BayesReg}_T(\pi;\P^\star) := \mathbb E_{P_{\pi_\star}}\!\big[R_T\big]$. The proof proceeds by first deriving a bound on \(\E_{P}[R_T]-\E_{Q}[R_T].\) First, we perform a change of measure by applying the Donsker--Varadhan variational formula (Lemma~\ref{lem:DV}) with $f=\lambda R_T$. Second, we apply a Hoeffding-type MGF bound for bounded adapted sums under $Q$ (Lemma \ref{lem:mgf-adapted}), and third, optimize over $\lambda$ to get \(E_P[R_T] - E_Q[R_T] \le B'\sqrt{\frac{T}{2}\,\KL(P\|Q)}\). Now, we note that Lemma~\ref{lem:B7-causal-chain-rule} relates the trajectory divergence to the causal action mismatch as $\KL(\P^\star_{\pi_\theta}\|\P^\star_{\pi_\star})=\Delta^{\mathrm{act}}_T(\theta)$. Moreover, Lemmas~\ref{lemma:B12-act-by-rev-comp}--\ref{cor:B13-rev-by-fwd-comp} bound the causal action mismatch by the reverse
completion mismatch, and then relate the reverse completion mismatch to the forward completion mismatch:
\(
\Delta^{\mathrm{act}}_T(\theta)
\le
\Delta^{\mathrm{comp}}_{T,\mathrm{rev}}(\theta;\pi_\theta)
\le
\bigl(2+\log V_{\mathrm{comp}}\bigr)\,\Delta^{\mathrm{comp}}_T(\theta;\pi_\theta)
\). This is done via the symmetry of the squared Hellinger distance, under Assumption \ref{assump:B-comp-overlap}. Combining this with the third step and using $E_P[R_T]=\BReg_T(\pi_\theta;\P^\star)$,
$E_Q[R_T]=\BReg_T(\pi_\star;\P^\star)$ yields the claimed result.

\begin{remark}[Interpretation of Theorem \ref{thm:sec25_deploy}]

This result shows that the additional price of using a learned completion model instead of the \textit{ideal} posterior-sampling benchmark $\pi_\star$, under the true environment, is a deployment penalty. This deployment penalty scales as \(\sqrt{T}\) times the square root of the on-policy completion mismatch. Importantly, note that in our theory, \emph{both} the deployed and the benchmark regret are defined under the \emph{same} true environment $\P^\star$. We depart from prior work \citep{zhang2025contextual} in this regard, whose benchmark is evaluated on the deployed history, i.e., counterfactual Bayesian regret. The rationale behind choosing this presentation is to have the benchmark regret equivalent to the TS on-policy regret. This provides a template to directly import any existing Bayesian regret results, thereby making our theory useful for application to a wide range of operational problems with known theoretical results. The overlap constant \(V_{\mathrm{comp}}\) makes explicit that severe support mismatch between \(Q_{t,\theta}\) and \(Q_t^\star\) can amplify this penalty. 
\end{remark}

%% file: sec3-repeated_nv.tex
\section{Methodology}
\label{sec:newsvendor}

In this section, we instantiate the generic ICGPS algorithm presented in Section \ref{sec:framework} to the particular case of the R-NV, by tailoring Algorithm \ref{alg:gps} to the particular requirements posed by the R-NV problem. We begin by listing two problem-specific obstacles that our method and analysis aim to address:

\emph{Obstacle 1: The per-round loss is continuous and not observable.} The per-round R-NV loss is convex and continuous, but not observable, and in general also not bounded. While the latter issue can be overcome by assuming a known bound $B>0$ such that $D_t\in[0,B]$ almost surely, the observability issue needs more careful analysis. We show that the R-NV admits a reduction to the feedback structure of bandit convex optimization (BCO) problem, and provide Bayesian regret guarantees using results from the BCO literature \cite{bakhtiari2025bco}.

\emph{Obstacle 2: Right-censoring creates challenges for training and sampling}. One must note that in the R-NV problem, the missing part of the data is not just unobserved future outcomes, but also the latent censored demand. This impacts \emph{both} the offline training objective and the online sampling routine:
\begin{itemize}
\item \textbf{Offline:} the appropriate likelihood is the likelihood of censored observations, not of raw demands.
This aligns with survival analysis: uncensored observations contribute a density term, whereas censored
observations contribute a survival term \citep{kaplan1958nonparametric}.
\item \textbf{Online:} autoregressive generation must respect the censoring constraints imposed by the observed
history; i.e., when $C_t=0$ the completion must satisfy $\widetilde D_t>X_t$.
\end{itemize}

Recent work on the censored newsvendor \citep{hssaine2024censored} shows that the aforementioned issues are significant, and censoring can create fundamental information loss and identifiability barriers unless sufficient coverage (large enough historical
orders) is present. To fix notation for later subsections, we record the censored likelihood identity. Lemma~\ref{lem:censored-likelihood} is the basic building block for the offline censoring-aware
negative log-likelihood used to train the completion model in Section~\ref{subsec:nv-offline},
and for the constraint-aware sampling procedure used online in Section~\ref{subsec:nv-online}.

\begin{lemma}[One-step censored likelihood]
\label{lem:censored-likelihood}
Fix a history $h$ and action $x\in\cX$.
Let $D$ be a demand random variable with conditional distribution (given $H_{t-1}=h$) admitting CDF $F(\cdot\mid h)$.
Assume $F(\cdot\mid h)$ is absolutely continuous with density $f(\cdot\mid h)$ on $(0,B)$.
Define the right-censored observation
\(
O=(S,C)=(\min\{D,x\},\mathbf{1}\{D\le x\}).
\)
Then the conditional density/mass of $O$ given $(h,x)$ is
\[
p_O(s,c\mid h,x)
=
\begin{cases}
f(s\mid h), & c=1,\; s<x,\\[2pt]
1-F(x\mid h), & c=0,\; s=x,\\[2pt]
0, & \text{otherwise.}
\end{cases}
\]
\end{lemma}

\begin{proof}
On the event $\{C=1\}$ we have $D\le x$ and $S=D$, so for $s<x$,
\(
\P(S\in ds,C=1\mid h,x)=\P(D\in ds\mid h)=f(s\mid h)\,ds.
\)
On the event $\{C=0\}$ we have $D>x$ and $S=x$, so
\(
\P(S=x,C=0\mid h,x)=\P(D>x\mid h)=1-F(x\mid h).
\)
All other $(s,c)$ pairs are impossible by definition of $(S,C)$.
\end{proof}

\subsection{In-context learning for R-NV}

In-context learning in modern generative models, including transformer-based architectures, involves two main steps: the first being offline meta-training, and the second being online in-context adaptation. Central to both these steps is the idea of tasks: episodes with their own demand law. The meta-training over tasks and in-context conditioning within a task is precisely the operational viewpoint in the missing-data Thompson sampling literature \citep{cai2024active, zhang2025contextual}.

\paragraph{Task.}
Let us denote by $\P$ the law of a latent task variable, with prior distribution $\xi$.
Conditional on $\P=\P^\star$, the demand sequence $(D_t)_{t=1}^T$ is generated from a (possibly unknown)
task-specific law $D_1,\dots,D_T \ \overset{\mathrm{i.i.d.}}{\sim} \P^\star$, where $\P^\star \sim \xi$. The conditional i.i.d.\ assumption is not necessary and is used only to obtain explicit likelihood factorizations later. Similar to before, the latent outcome at time $t$ is $D_t$, and the complete trajectory is $D_{1:T}$. The learner observes neither $\P^\star$ nor $D_t$, but only $O_t=\psi(X_t,D_t)$.

\paragraph{Bayes risk within a task and Bayes regret.}
For a realized task $\P=\P^\star$, define the (task-conditional) Bayes risk $f_{\P^\star} := \E_{\P^\star}\!\left[\ell(x,D_t)\right]$, where the RHS does not depend on $t$ in the conditionally i.i.d.\ (or stationary) case.
Let $x_{\P^\star}^\star\in\arg\min_{x\in\cX} f_{\P^\star}(x)$. Omitting the subscripts for brevity, the Bayesian regret is the expectation over the prior on $\P^\star$ and all randomness of the interaction:
\(
\BReg_T(\pi;\P^\star)
:= \E\!\left[\sum_{t=1}^T \bigl(f(X_t)-f(x^\star)\bigr)\right]
\), similar to the definition in Section \ref{sec:nv_running_example}.

\paragraph{Offline meta-training and Online in-context adaptation.}
Offline, the learner is given a dataset consisting of $N$ historical tasks (episodes)
\(\cD_{\mathrm{off}} := \bigl\{ H^{(i)}_{T_i} \bigr\}_{i=1}^N, H^{(i)}_{T_i} = \bigl((X^{(i)}_t, O^{(i)}_t)\bigr)_{t=1}^{T_i},\)
where each episode is generated by first sampling $\P^{\star}_i\sim \xi$ and then generating demands
$D^{(i)}_{1:T_i}\sim \P^{\star}_i$ and interacting under some (possibly unknown) \emph{behavior policy}
that selects $X^{(i)}_t$ and reveals only $O^{(i)}_t=\psi(X^{(i)}_t,D^{(i)}_t)$.
We do \emph{not} assume the behavior policy is optimal or even stationary; it is part of the offline data
collection mechanism. This offline dataset is used to \emph{meta-learn} a completion model (Section~\ref{subsec:nv-offline}). Online, the learner is evaluated on a fresh task $\P^\star \sim \xi$ and adapts by conditioning on the
within-task history $H_{t-1}$ (no online weight updates). This is exactly the \emph{in-context} usage
of the learned completion model described abstractly in Section~\ref{subsec:gen-ps}.

\subsection{Offline meta-learning: training completion model under censored demand}
\label{subsec:nv-offline}

This subsection specifies the \emph{offline} learning problem used to fit the completion model
$p_\theta(\cdot\mid h)$ required by generative posterior sampling. The key issue is that the offline data are \emph{right-censored}:
we observe $(X_t,O_t)$ with $O_t=(S_t,C_t)$ but generally do not observe the latent demand $D_t$
on stockout rounds (refer Obstacle 2 above). The goal of this section is to provide a \emph{trainable} conditional model of demand given a within-task history, which can then be used online to sample completions.

\paragraph{In-context training via random prefixes.}
Recall the task variable $\P^\star \sim \xi$. Offline we observe $N$ independent \emph{episodes} (tasks) $\cD_{\mathrm{off}} := \bigl\{ (H^{(i)}_{T_i}) \bigr\}_{i=1}^N$, as defined before. To reflect how the completion model will be used online (conditioning on a within-task prefix $H_{t-1}$),
we view each episode as supplying many training pairs
\[
\bigl(\text{prompt } H^{(i)}_{t-1},\ \text{next observation } O^{(i)}_t\bigr),
\qquad t=1,\dots,T_i.
\]

\paragraph{Model class: conditional law of  demand}
We posit a parametric family $\{p_\theta:\theta\in\Theta\}$ of conditional laws for $D_t$ given history.
Formally, for each $\theta\in\Theta$ and each history value $h_{t-1}$, let \(p_\theta(\cdot\mid h_{t-1})\) be a probability measure on $[0,B]$, with associated conditional CDF $F_\theta(x\mid h_{t-1}) := p_\theta(D_t\le x\mid h_{t-1})$ and (when it exists) density $f_\theta(x\mid h_{t-1}) := \frac{\partial}{\partial x} F_\theta(x\mid h_{t-1})$.
We allow $p_\theta(\cdot\mid h)$ to depend on the full observed history $h$ (including past actions),
which is natural for sequence models; this dependence is crucial for in-context adaptation.

\begin{remark}[Autoregressive sequence modeling viewpoint]
One can implement $p_\theta(\cdot\mid h_{t-1})$ by an autoregressive sequence model trained by next-token
prediction (e.g., a Transformer), where the input tokens encode the prompt $h_{t-1}$ and the output head
parameterizes a distribution over $D_t$.
For the theory below, we only need that the model outputs a valid CDF $F_\theta(\cdot\mid h)$ (and possibly a
density $f_\theta(\cdot\mid h)$).
\end{remark}

\subsubsection{Censoring-aware observed-data likelihood}

The offline dataset does not reveal $D_t$ when $C_t=0$ (stockout). Thus, we cannot train $p_\theta$
by direct demand log-likelihood. Instead, we maximize the likelihood of the \emph{censored observation}
$O_t=(S_t,C_t)$ induced by $p_\theta$. Fix a history $h_{t-1}$ and action $x\in[0,B]$.
Under a conditional law $p_\theta(\cdot\mid h_{t-1})$ for $D_t$, define the induced conditional law for
$O_t=(S_t,C_t)$ by the pushforward through the censoring map
\(
\psi(x,D)=(\min\{D,x\},\ind{D\le x}).
\)
When $p_\theta(\cdot\mid h_{t-1})$ has density $f_\theta(\cdot\mid h_{t-1})$, Lemma~\ref{lem:censored-likelihood}
specializes to:
\begin{equation}
\label{eq:nv-induced-obs-lik}
p_{\theta,\mathrm{obs}}(s,c\mid h_{t-1},x)
=
\begin{cases}
f_\theta(s\mid h_{t-1}), & c=1,\; s<x,\\[2pt]
1-F_\theta(x\mid h_{t-1}), & c=0,\; s=x,\\[2pt]
0, & \text{otherwise.}
\end{cases}
\end{equation}
This is the standard right-censoring likelihood: uncensored observations contribute a density term and
censored observations contribute a survival term \citep{kaplan1958nonparametric,kalbfleisch2002,klein1997}. The oracle requires completed demands, but censored observations only reveal whether $D_t$ lies above or below $X_t$ and, if below, its exact value. Training on the induced observation likelihood \eqref{eq:nv-induced-obs-lik} is therefore the
\emph{maximum-likelihood} way to learn the conditional demand law from censored data. Define the one-step predictive log-loss at time $t$ as
\begin{equation}
\label{eq:nv-one-step-nll}
\ell^{\mathrm{obs}}_\theta\bigl(O_t\mid H_{t-1},X_t\bigr)
:=
-\log p_{\theta,\mathrm{obs}}\bigl(O_t\mid H_{t-1},X_t\bigr),
\end{equation}
with $p_{\theta,\mathrm{obs}}$ given by \eqref{eq:nv-induced-obs-lik}. Writing $O_t=(S_t,C_t)$, we obtain the censoring-aware negative log-likelihood
\[
\ell^{\mathrm{obs}}_\theta(O_t\mid H_{t-1},X_t)
=
\begin{cases}
-\log f_\theta(S_t\mid H_{t-1}), & C_t=1\ \ (\text{uncensored}),\\[2pt]
-\log\bigl(1-F_\theta(X_t\mid H_{t-1})\bigr), & C_t=0\ \ (\text{censored}).
\end{cases}
\]

\paragraph{Population and empirical objective}

We now formalize what the offline objective estimates at the population level.
Fix any (possibly stochastic, history-dependent) behavior policy used to collect offline data.
Let $\P^\star_{\pi}$ denote the induced law of offline episodes under $\P^\star$ and that behavior policy.
Define the population censored predictive risk
\begin{equation}
\label{eq:nv-pop-risk}
\cL^{\mathrm{obs}}(\theta)
:=
\E_{\P^\star_{\pi}}\!\left[
\sum_{t=1}^{T} \ell^{\mathrm{obs}}_\theta\bigl(O_t\mid H_{t-1},X_t\bigr)
\right],
\end{equation}
(where $T$ is the episode length; for variable-length episodes use $\sum_{t=1}^{T_i}$ and average over $i$).
Let $\cL^{\mathrm{obs}}(\star)$ denote the same quantity evaluated at the true induced observation kernel
$\P^\star(O_t\in\cdot\mid H_{t-1},X_t)$.
Given offline episodes $\cD_{\mathrm{off}}=\{H^{(i)}_{T_i}\}_{i=1}^N$, we minimize the empirical version of
\eqref{eq:nv-pop-risk}:
\begin{equation}
\label{eq:nv-emp-risk}
\widehat \cL^{\mathrm{obs}}_N(\theta)
:=
\frac{1}{N}\sum_{i=1}^N
\sum_{t=1}^{T_i}
\ell^{\mathrm{obs}}_\theta\!\Bigl(O^{(i)}_t\ \big|\ H^{(i)}_{t-1},X^{(i)}_t\Bigr).
\end{equation}
A minimizer $\hat\theta\in\arg\min_{\theta\in\Theta} \widehat \cL^{\mathrm{obs}}_N(\theta)$ defines a learned
completion model $p_{\hat\theta}(\cdot\mid h)$.
A practical implementation samples random prefixes (prompts) $H^{(i)}_{t-1}$ from offline episodes and trains the
model to predict the next censored observation $O^{(i)}_t$. The objective is exactly the stochastic optimization
of \eqref{eq:nv-emp-risk}. The model is trained to perform conditional prediction given a prompt,
and at deployment time it adapts by conditioning on the online prompt rather than by updating parameters.

\subsection{Online algorithm: in-context GPS for R-NV via autoregressive sampling}
\label{subsec:nv-online}

This subsection specializes ICGPS (Algorithm~\ref{alg:gps}) to the repeated newsvendor with right-censored demand. The primary additional ingredient beyond Section~\ref{sec:framework} is that
the completion sampler must respect censoring constraints induced by the observed history. Throughout, we work with the bounded loss $\ell$ on $\cX=[0,B]$ and the oracle map $\hat x(\cdot)$ which is the empirical $\gamma$-quantile.

\paragraph{Completion variables and censoring constraints.} Recall that the latent outcome is the demand $D_t\in[0,B]$, and the observation is $O_t=(S_t,C_t)=\psi(X_t,D_t)=\bigl(\min\{D_t,X_t\},\ \ind{D_t\le X_t}\bigr)$. We treat $D_t$ as the completion variable. The history at time $t$ is \( H_t=\bigl((X_s,O_s)\bigr)_{s=1}^t\). Fix a realized history value $h_{t-1}=\bigl((x_s,(s_s,c_s))\bigr)_{s=1}^{t-1}$.
Any completed demand trajectory $D_{1:T}=(D_1,\dots,D_T)\in[0,B]^T$ is
\emph{consistent} with $h_{t-1}$ if for each $s<t$:
\begin{equation}
\label{eq:consistency-constraints}
\begin{cases}
D_s = s_s, & \text{if } c_s=1 \quad (\text{no stockout}),\\
D_s \in (x_s,B], & \text{if } c_s=0 \quad (\text{stockout}).
\end{cases}
\end{equation}

\subsubsection{Learned completion kernel and constrained sampling}

Section~\ref{subsec:nv-offline} defines a learned conditional model $p_\theta$ for  demand.
Online GPS requires sampling a \emph{completion} from a conditional law over entire trajectories:
\(\widetilde{D}_{1:T}^{(t)} \sim p_\theta(\cdot\mid H_{t-1})\). In practice, $p_\theta(\cdot\mid H_{t-1})$ is implemented by autoregressive generation conditioned on the prompt
$H_{t-1}$ (\citealp{cai2024active,zhang2025contextual}), with additional conditioning that enforce the hard
constraints \eqref{eq:consistency-constraints}.

\paragraph{Implementation via autoregressive factorization.}
We assume that, for each prompt $h_{t-1}$, the model admits an implementable factorization of the form
\begin{equation}
\label{eq:ar-factor}
p_\theta(D_{1:T}\mid h_{t-1})
=
\prod_{s=1}^T p_\theta\!\left(D_s \,\middle|\, D_{1:s-1},\, h_{t-1}\right),
\end{equation}
where each one-step conditional is samplable. This is exactly the interface provided by autoregressive sequence models: the prefix $D_{1:s-1}$ and the prompt $h_{t-1}$ are treated as ``context tokens,'' and the model outputs a distribution for the next value. When $s<t$ and $c_s=1$, the constraint $D_s=s_s$ is deterministic. When $s<t$ and $c_s=0$ with $x_s<B$, we
must sample $D_s$ conditional on the event $D_s>x_s$. The next lemma gives the exact conditional law. Lemma~\ref{lem:tail-conditioning} yields an exact, \emph{constraint-aware} sampling routine: sample $D_s$ from the one-step conditional \eqref{eq:ar-factor} when unconstrained, and sample from its tail conditional law \eqref{eq:tail-cdf} when the history indicates a stockout.

\begin{lemma}[Tail conditioning for  demand]
\label{lem:tail-conditioning}
Fix a prompt $h_{t-1}$ and a prefix $D_{1:s-1}$.
Let $\Pi(\cdot)$ denote the conditional law
\(
\Pi(\cdot) = p_\theta(D_s\in\cdot\mid D_{1:s-1}=D_{1:s-1},\,h_{t-1})
\)
on $[0,B]$, and write its CDF as $F(z)=\Pi([0,z])$.
Fix a threshold $x\in[0,B)$ such that $\Pi(D_s>x)>0$.
Then the conditional law of $D_s$ given $\{D_s>x\}$ is supported on $(x,B]$ and satisfies
\begin{equation}
\label{eq:tail-cdf}
\Pi(D_s\le z \mid D_s>x)
=
\frac{F(z)-F(x)}{1-F(x)},
\qquad z\in[x,B].
\end{equation}
If $\Pi$ admits a density $f$ on $(0,B)$, then the conditional density on $(x,B)$ is
\begin{equation}
\label{eq:tail-density}
f^{\mathrm{tail}}(z)
=
\frac{f(z)}{1-F(x)}\ \ind{z>x}.
\end{equation}
\end{lemma}

\begin{proof}
For $z\in[x,B]$,
\[
\Pi(D_s\le z \mid D_s>x)
=
\frac{\Pi(x<D_s\le z)}{\Pi(D_s>x)}
=
\frac{\Pi(D_s\le z)-\Pi(D_s\le x)}{1-\Pi(D_s\le x)}
=
\frac{F(z)-F(x)}{1-F(x)}.
\]
The density form follows by differentiating \eqref{eq:tail-cdf} on $(x,B)$.
\end{proof}

\paragraph{ICGPS Algorithm.} We can now state the online policy. At each round $t$, the algorithm: (i) samples a completion consistent with the observed censoring pattern, (ii) computes the oracle action (empirical $\gamma$-quantile) for that completion, and (iii) observes a new censored datum and appends it to the prompt. Algorithm~\ref{alg:nv-gps} performs \emph{no} online updates to $\theta$. All adaptation to the new (unknown) task is through conditioning on the prompt $H_{t-1}$ (in-context adaptation). If the one-step conditional in \eqref{eq:ar-factor} provides an invertible CDF $F_\theta(\cdot\mid \cdot)$, then tail sampling can be done by inverse transform: draw $U\sim \mathrm{Unif}(F_\theta(X_s\mid \cdot),1)$ and set $\widetilde{D}_s^{(t)} = F_\theta^{-1}(U\mid \cdot)$. Otherwise, one may use rejection sampling (sample from $p_\theta$ until $>X_s$) or any exact sampler for distributions. Our proposed architecture in the experimental section supports the former.

\subsection{Theoretical analysis of ICGPS for repeated newsvendor}
\label{subsec:nv-theory}

The Bayesian regret of the deployed ICGPS policy for the R-NV problem (Algorithm \ref{alg:nv-gps}) admits the decomposition into the two terms, put forward by Theorem \ref{thm:sec25_deploy}. In this section, we tailor this decomposition given by Equation \eqref{eq:sec25_deploy} to the specific feedback structure of the R-NV. In \S\ref{sec:rnv_bayes_regret}, we upper bound the first term by analyzing the Bayesian regret of Thompson sampling for the R-NV problem. This involves reducing our censored feedback structure to the feedback structure of bandit convex optimization. In \S\ref{subsubsec:nv-link-deploy-to-obs} we upper bound the second term (completion mismatch) by the \emph{observable} predictive objective used to train the completion model from censored data. This involves deriving assumptions on coverage and self-contraction (a form of harmonic decay) to account for the information loss due to censoring. In \S\ref{sec:final_regret_12} we combine the previous two bounds to obtain the final regret guarantee for ICGPS in the repeated newsvendor setting.

\subsubsection{Term I: Bayesian regret of Thompson sampling for repeated newsvendor}
\label{sec:rnv_bayes_regret}
We recall that $\xi$ denotes the Bayesian prior over $\P^\star$, and that $\xi(\cdot\mid H_{t-1})$ is the posterior given history. The oracle policy $\pi^\star=\pi_{\P^\star}$ implements Thompson sampling by drawing a posterior sample $\widetilde \P_t\sim \xi(\cdot\mid H_{t-1})$ and then playing
\(X_t \in \arg\min_{x\in[0,B]} f_{\widetilde \P_t}(x)\). Equivalently (Appendix~\ref{subsec:continuous}), one may view TS as sampling a \emph{random function} $f^{(t)}\sim \Law(f^\star\mid H_{t-1})$, and playing its minimizer: this function-sampling perspective provides a natural interface for continuous-action problems, and allows us to borrow ideas from the literature on bandit convex optimization.

\begin{theorem}[Bayesian regret of TS on R-NV]
\label{thm:rnv_bayes_regret}
Assume the demand is bounded: $D_t\in[0,B]$ almost surely for all $t$. Let $\pi_\star=\pi_{\P^\star}$ be the ideal GPS policy that uses the true completion kernel, i.e., $\pi_\star$ implements Thompson Sampling for the R-NV problem. Then
\begin{equation}
\label{eq:nv_bayes_regret}
\BReg_T(\pi_\star;\P^\star)
\le
\widetilde \cO (\sqrt{T})\,,
\end{equation}
where the notation $\widetilde{O}(\cdot)$ is used to hide polylogarithmic factors (and constants depending on $B,h,b$).
\end{theorem}

\paragraph{Proof Sketch.} We provide a proof sketch here, and refer the reader to Appendix \ref{app:proof_thm_ts_newsvendor} for the full proof. The proof hinges on the fact that the R-NV problem admits a derived bandit feedback that is equivalent to the feedback in bandit convex optimization (BCO). We introduce a risk-equivalent convex objective $g_\P$ (Lemma~\ref{lem:shift_objective}) such that $f_\P(\cdot)$ and $g_\P(\cdot)$ yield identical regret, so it suffices to bound
$\E\big[\sum_{t\le T}\big(g_\P(X_t)-g_\P(x^\star_\P)\big)\big]$. We construct an \emph{unbiased} (Lemma \ref{lem:unbiased_feedback}) and \emph{observable} (Lemma \ref{lem:Z_computable}) feedback for $g_\P$ given by $Y_t :=(h+b)X_t Z_t - bX_t$. This is done by augmenting the learner with auxiliary randomness $U_t\sim\mathrm{Unif}([0,X_t])$ drawn independently each round, and defining the latent indicator $Z_t:=\ind{D_t\le U_t}$. Finally, showing the required regularity (convexity, Lipschitzness and boundedness) using Lemma \ref{lem:g_properties}, and scaling gives a valid BCO instance with scalar actions, yielding the $\widetilde \cO(\sqrt{T})$ regret bound from \cite[Theorem 4]{bakhtiari2025bco}.

\begin{remark}[Interpretation of Theorem \ref{thm:rnv_bayes_regret}]
This theorem shows that the Thompson sampling algorithm for the R-NV problem is able to achieve sublinear Bayesian regret, despite censored feedback. To the best of our knowledge, this is the first proof applying the information ratio analysis \cite{russo2016info} to obtain the Bayesian regret of TS for the R-NV problem. The key technical novelties include deriving the computable unbiased feedback structure and recognizing the equivalence of function and outcome sampling via Lemma \ref{lem:outcome-vs-function}. 
\end{remark}

\subsubsection{Term II: Linking deployment penalty to predictive objective}
\label{subsubsec:nv-link-deploy-to-obs}
The second term in Theorem~\ref{thm:sec25_deploy} is governed by the on-policy completion mismatch
$\Delta^{\mathrm{comp}}_T(\theta;\pi_\theta)$, which is not directly observable under right-censoring. In this subsection, we show that for R-NV, the completion mismatch is controlled by the censored predictive objective.
Let us denote the observed-data kernels, i.e., the conditional laws of $O_t$ given $(H_{t-1},X_t)=(h,x)$ under the true model and under $\theta$, as $P^{\mathrm{obs}}_{h,x}$ and $Q^{\mathrm{obs}}_{h,x}$, respectively. Alternatively, they can also be defined as pushforwards through $\psi$:
\(
P^{\mathrm{obs}}_{h,x}
\;:=\;
\bigl(q_t^\star(\cdot\mid h)\bigr)\circ \psi(x,\cdot)^{-1},
Q^{\mathrm{obs}}_{h,x}
\;:=\;
\bigl(q_{t,\theta}(\cdot\mid h)\bigr)\circ \psi(x,\cdot)^{-1}
\). We measure \emph{censoring-aware predictive mismatch} along a policy $\pi$ by the KL mismatch between observed-data kernels as below
\begin{equation}
\Delta^{\mathrm{obs}}_T(\theta;\pi)
\;:=\;
\sum_{t=1}^T
\mathbb{E}_{(H_{t-1},X_t)\sim \P^\star_\pi}
\!\left[
\KL\!\Bigl(P^{\mathrm{obs}}_{H_{t-1},X_t}\,\Big\|\,Q^{\mathrm{obs}}_{H_{t-1},X_t}\Bigr)
\right].
\label{eq:def-Delta-obs}
\end{equation}
Equivalently, if $p_{\theta,\mathrm{obs}}(\cdot\mid h,x)$ denotes the model likelihood of $O_t$ induced by $q_{t,\theta}$
and $\psi$, then $\Delta^{\mathrm{obs}}_T(\theta;\pi)$ coincides with the population excess censored log-loss, due to Lemma \ref{lem:delta-obs-excess-Lobs} (Appendix \ref{apx:lemmalist_thm3}). Therefore, defining $\cL^{\mathrm{obs}}_T(\theta;\pi$ and $\cL^{\mathrm{obs}}_T(\star;\pi)$ according to \eqref{eq:nv-pop-risk}, we have
\begin{equation}
\Delta^{\mathrm{obs}}_T(\theta;\pi)
\;=\;
\cL^{\mathrm{obs}}_T(\theta;\pi)-\cL^{\mathrm{obs}}_T(\star;\pi).
\label{eq:Delta-obs-excess-logloss}
\end{equation}

Because censoring discards information about the demand tail, $\Delta^{\mathrm{obs}}_T(\theta;\pi)$ does not, in general,
control $\Delta_T^{\mathrm{comp}}(\theta;\pi_\theta)$ without further assumptions. In the following, we provide a minimal \emph{coverage} condition, in line with the identifiability barriers highlighted in the censored newsvendor literature \citep{hssaine2024censored}. Moreover, we also need an assumption on harmonic decay along the AR pseudo-generation. This is given by a relation between the one-step divergence $d(h)$ and $s-$th conditional divergence $\delta_s(h)$, between completion kernels $Q_t^\star(\cdot\mid h)$ and $Q_{t,\theta}(\cdot\mid h)$, respectively. The formal definitions of these two quantities, along with further details about implications of this assumption can be found in Appendix \ref{apx:thm_3_asm}.

\begin{theorem}[Relation between deployment penalty and censoring-aware predictive mismatch]
\label{thm:completion_to_ovserved}
Consider the deployed policy $\pi_\theta$ in the true environment $\P^\star$.
Assume that, for $\P^\star_{\pi_\theta}$-a.e.\ history $h$ and all $t$ the following hold:
\begin{enumerate}
\item Max-order coverage: $\P^\star_{\pi_\theta}(X_t=B\mid H_{t-1}=h)\ge \eta$, for some $\eta>0$; and
\item Self-contraction along AR trajectory: $\delta_s(h)\le (c_{\mathrm{sc}}/s)\,d(h)$, for all $s\ge 2$.
\end{enumerate}

Then the on-policy completion mismatch is upper bounded by the observed censoring-aware predictive mismatch:

\begin{equation}
\label{eq:completion_to_ovserved}
\Delta_T^{\mathrm{comp}}(\theta;\pi_\theta)
\;\le\; \frac{1 + c_{sc} \log T}{\eta}\,\Delta_T^{\mathrm{obs}}(\theta;\pi_\theta).
\end{equation}
\end{theorem}

\paragraph{Proof Sketch.}
We provide a proof sketch here, with the main technical components, and refer the reader to Appendix~\ref{app:proof-thm-4.8} for the full proof. Using a KL chain-rule decomposition for AR-completion kernels, we show that for each $(t,h)$,
\(
\KL\!\Bigl(Q_t^\star(\cdot\mid h)\,\big\|\,Q_{t,\theta}(\cdot\mid h)\Bigr)
\le
\Bigl(1+c_{\mathrm{sc}}\log T\Bigr)\,d(h),
\)
which is summarized as Lemma~\ref{lem:self-contraction-bound}. When $x=B$, the censoring map is (a.s.) information-preserving under (i), so the latent divergence $d(h)$
can be upper bounded by the observed-data KL averaged over the deployed action distribution.
The key inequality is Lemma~\ref{lem:coverage-obs-to-latent}, yielding
\(
d(h)
\le
\frac{1}{\eta}\,
\mathbb{E}_{X_t\sim \pi_{\theta,t}(\cdot\mid h)}
\!\left[
\KL\!\Bigl(P^{\mathrm{obs}}_{h,X_t}\,\big\|\,Q^{\mathrm{obs}}_{h,X_t}\Bigr)
\right].
\)
Take expectations over $H_{t-1}\sim \P^\star_{\pi_\theta}$, sum over $t\le T$, and recognize the right-hand side
as $\Delta^{\mathrm{obs}}_T(\theta;\pi_\theta)$ via \eqref{eq:def-Delta-obs}.
Finally, we can use the excess-log-loss identity \eqref{eq:Delta-obs-excess-logloss} (Lemma~\ref{lem:delta-obs-excess-Lobs}) to connect
$\Delta^{\mathrm{obs}}_T$ to the observable censored predictive objective.

\begin{remark}[Interpretation of Theorem \ref{thm:completion_to_ovserved}]
This theorem operationalizes the deployment-penalty of Theorem~\ref{thm:sec25_deploy} by relating it to the observed mismatch via two key assumptions.
The two assumptions have clear roles: the $1/\eta$ factor is an identifiability price: without a persistent probability of the de-censoring action $X_t=B$,
the tail of the demand distribution is not statistically constrained by the data, and observed predictive fit need not imply latent fit. On the other hand, the logarithmic factor arises from controlling error propagation along the AR pseudo-generation through
self-contraction. We need the latter assumption, while \citet{zhang2025contextual} do not, since our theoretical statement with on-policy TS Bayesian regret in the first term requires us to compare two different trajectory laws under $\P^\star_{\pi_\theta}$ and $\P^\star_{\pi_\star}$, respectively.
\end{remark}

\subsubsection{Final regret of ICGPS for repeated newsvendor}
\label{sec:final_regret_12}

\begin{corollary}[Bayesian regret of ICGPS for R-NV]
\label{cor:nv-final-regret-obs}
Assume that the assumptions in Theorems \ref{thm:sec25_deploy}, \ref{thm:rnv_bayes_regret}, and \ref{thm:completion_to_ovserved}, hold. Then, we have the following Bayesian regret bound for in-context GPS (Algorithm \ref{alg:nv-gps}) on the repeated newsvendor
\[
\BReg_T(\pi_{\theta};\P^\star)
\;\le\;
\widetilde \cO(\sqrt{T})
+ \Psi\left(B^\prime, V_{\mathrm{comp}}, \eta, c_{sc}\right)
\sqrt{\left(T \log T \right)\Delta_T^{\mathrm{obs}}(\theta;\pi_\theta)}\,,
\]
where $\Psi\left(B^\prime, V_{\mathrm{comp}}, \eta, c_{sc}\right)$ collects all the constant factors.
\end{corollary}

\begin{proof}
The proof can be done by simply plugging in the results of Theorems \ref{thm:rnv_bayes_regret} and \ref{thm:completion_to_ovserved}, into the RHS of Theorem \ref{thm:sec25_deploy}. Finally we can collect the constant factors into $\Psi\left(B^\prime, V_{\mathrm{comp}}, \eta, c_{sc}\right)$ to obtain the desired result.
\end{proof}

%% file: sec4-experiments.tex
\section{Experimental Evaluation}
\label{sec:experiments}
Our experimental evaluation addresses four empirical questions aligned to our methodological and theoretical contributions in Sections~\ref{sec:framework}--\ref{sec:newsvendor}: (Q1) As a first sanity check for our proposed ICGPS, we check whether it matches Thompson sampling with  correctly specified prior under i.i.d.\ demand?
(Q2) Secondly, we investigate whether offline quality translates into online regret, i.e., does improving the completion model reduce regret in a fixed environment?
(Q3) Thirdly, does multi-task pretraining improve robustness and transfer learning capability under distribution shift and prior mismatch?
(Q4) Finally, for real-world data, does in-context GPS improve performance on non-i.i.d.\ demand streams with temporal structure? We begin this section by describing our experimental setup, followed by our proposed ChronosFlow architecture for implementing the ICGPS algorithm in practice. Finally, we present the experiments to answer the above questions, and our inferences.

\subsection{Experimental setup}
This section fixes a common
evaluation protocol used throughout Section~\ref{sec:experiments}. We adhere to the repeated
newsvendor interaction model and notation from Section~\ref{sec:framework}. The section ends with a discussion about data generating processes, for both online environments and offline tasks for meta-training.

\begin{itemize}[leftmargin=*]
\item
\emph{Performance metrics.}
\label{subsubsec:exp-metrics}
We evaluate all methods using the Bayesian regret defined in Sections~\ref{sec:framework} and \ref{sec:newsvendor}.
For online performance summaries from the regret trajectory $\{\BReg_t\}_{t=1}^T$ we report regret curves $\BReg_t$ as a function of $t$. For non-iid experiments with real-world data, we cannot employ Bayesian regret, since the oracle is unknown in this case. Therefore, we report the newsvendor cost.

\item
\emph{Tasks and statistical reporting.}
A \emph{trial} is one full horizon-$T$ interaction with a fixed environment $\P^\star$ (i.e., one repeated-newsvendor task). Unless stated otherwise, synthetic experiments use a common horizon $T$, and each figure compares methods on the same $T$. 
Within each trial $i\in\{1,\dots,N_{\mathrm{trial}}\}$ we use a fixed latent demand stream $D_{1:T}^{(i)}$ and run all methods on this shared stream.
For any method $\mathsf{A}$ and time $t$, let
\(
R_t^{(i)}(\mathsf{A})
:= \sum_{s=1}^t \Big(\ell(X_s^{(i)},D_s^{(i)})-\ell(x^{\star,(i)},D_s^{(i)})\Big)
\)
denote the realized cumulative regret in trial $i$.
We report the Monte Carlo mean
\(\widehat{\mu}_t(\mathsf{A}) := \frac{1}{N_{\mathrm{trial}}}\sum_{i=1}^{N_{\mathrm{trial}}} R_t^{(i)}(\mathsf{A})\) and visualize uncertainty using $\pm 1$ standard error bands, where
\(
\widehat{\mathrm{se}}_t(\mathsf{A})
:= \sqrt{\frac{1}{N_{\mathrm{trial}}(N_{\mathrm{trial}}-1)}
\sum_{i=1}^{N_{\mathrm{trial}}}\Big(R_t^{(i)}(\mathsf{A})-\widehat{\mu}_t(\mathsf{A})\Big)^2 }.
\) 
Trials are generated based on different random seeds.

\item 
\emph{Online environments.}
For the first three experiments, we evaluate all methods online on synthetic repeated newsvendor tasks with a
controlled set of demand families (described in more detail in the individual experiments below \S\ref{subsec:exp1}-\ref{subsec:exp3}). We set $h=1$ and vary the critical fractile (also called service level in operations management literature) $\gamma=b/(b+h)$, by varying $b$, over $\gamma\in\{0.5,0.9,0.98\}$. Censoring severity increases as $\gamma$ decreases. For the last experiment with real-world data, we follow the setup of \citealt{hssaine2024censored}, details in Section \ref{subsec:exp4}.
\item 
\emph{Offline meta-training data.}
Meta-training data are generated from the same synthetic task families described above. Offline, the learner is given historical episodes (tasks) generated under some data-collection mechanism: the behavior policy need not be optimal or stationary. For synthetic experiments, we construct an offline corpus by sampling independent tasks and then simulating right-censored trajectories, as follows: (i) sample a task $\P^\star \sim \xi$, and roll out a length-$T$ trajectory by repeatedly drawing $D_t \sim \P^\star$, (ii) choose orders via an \emph{exploration} policy $\mu$ that covers the action space, and (iii) record only the censored observations $O_t=\psi(X_t,D_t)$.
\end{itemize}

\subsection{Implementation of ICGPS for R-NV}
\label{subsec:exp-implementation}
We propose the ChronosFlow-ICGPS architecture (Figure \ref{fig:chronos-2-arch}) to instantiate the learned completion kernel in Algorithm~\ref{alg:nv-gps}. The goal is to produce
fast conditional samples of the latent demand process while enforcing censoring constraints induced by the
censoring map, via offline meta-training and online in-context inference. The architecture consists of three main components: (a) conditioning context vector layer, (b) conditional normalizing flow (CNF) head, and (c) ICGPS sampler. We explain the main highlights here, and defer some of the details for Appendix \ref{app:exp-implementation}.
\begin{itemize}[leftmargin=*]
\item \emph{(a) Conditioning context vector layer}. This layer is used to produce the context vector $\widebar H_{t-1} = \omega(H_{t-1})$ during both offline meta-training and online sampling. It consists of a Chronos backbone that produces an in-context quantile grid at levels \(\{0.05,0.10,\ldots,0.95\}\), and a Kaplan--Meier (KM) module that maintains censoring-consistent quantiles on the same grid. The final context vector is obtained by concatenating the above Chronos quantiles and KM quantiles to a simple
summary statistic of the historical data. 
During offline meta-training, the Chronos backbone can be either frozen (real-data experiments) or fine-tuned (i.i.d.\ experiments). During online sampling, this layer is always frozen.
\item \emph{(b) Conditional normalizing flow head.} The CNF head consists of a light hypernetwork that maps the above context vector to the parameters of a one-dimensional strictly monotone conditional normalizing flow, finally outputting the induced CDF $F_\theta (\cdot \mid \widebar H_{t-1})$ and inverse CDF $F^{-1}_\theta(\cdot \mid \widebar H_{t-1})$. It is optimized during offline training via the censoring-aware negative log-likelihood objective \eqref{eq:nv-one-step-nll}, and frozen during online sampling.

\item \emph{(c) ICGPS sampler.}
Online sampling is performed according to the description in \S\ref{subsec:nv-online}, and Algorithm \ref{alg:nv-gps}. In case of censoring, we sample from the conditional tail \(p_\theta(\cdot\mid \widebar h_s,\,D_s>\widebar H_s)\) using exact inverse-CDF tail-conditioning $U\sim \mathrm{Unif}\!\big(F_\theta(\widebar H_s\mid \widebar h_s),1\big), \widetilde D_s := Q_\theta(U\mid \widebar h_s).$ At each time \(t\), we draw \(M\) independent completions
\(\{\widetilde D^{(m)}_{1:T}\}_{m=1}^M\) from the constrained decoder, compute the corresponding newsvendor
oracle action \(\widehat x(\widetilde D^{(m)}_{1:T})\), and play a robust
aggregate (median over \(m\)) as \(X_t\). This aggregation stabilizes decisions against occasional
outlier completions. We use a short warm-up of \(T_0=3\) rounds (e.g.,
ordering \(B\)) before switching to the rollout-based policy.
\end{itemize}

\begin{figure}[t]
    \centering
    \includegraphics[width=0.95\linewidth]{"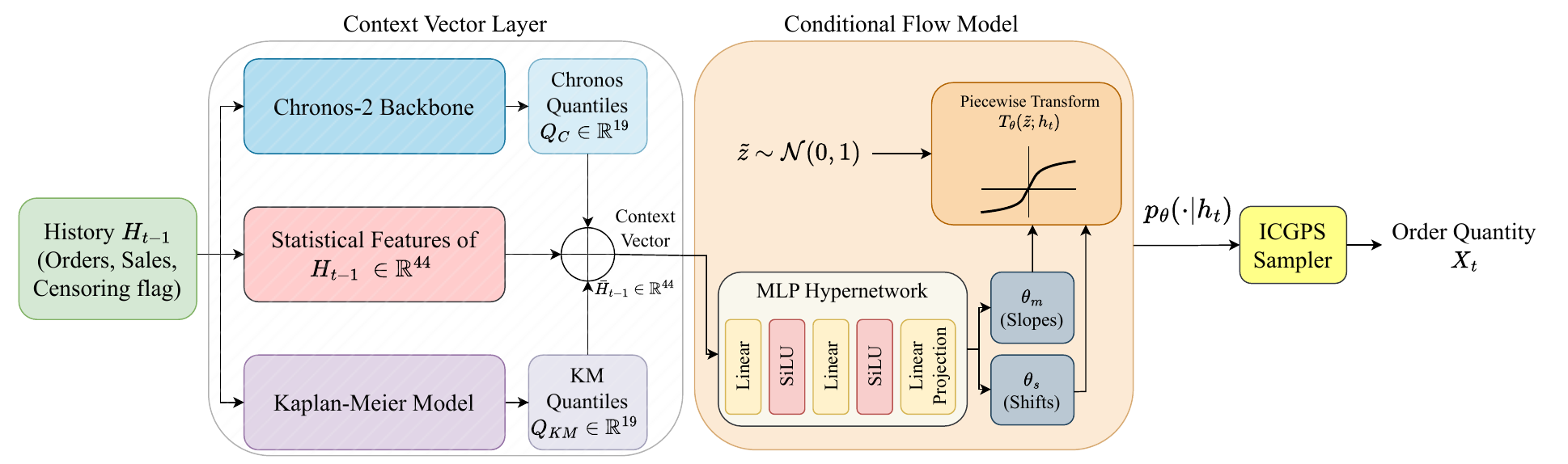"}
    \caption{ChronosFlow-ICGPS Architecture}
    \label{fig:chronos-2-arch}
\end{figure}

\subsection{Experiments}

\subsubsection{(Q1) Sanity check}
\label{subsec:exp1}
This experiment is a sanity check for our ICGPS methodology against Thompson sampling with correctly specified prior. The demand is given by \(D_t \stackrel{\text{i.i.d.}}{\sim} \mathrm{Weibull}(k,\theta^\star)\) with CDF
\(F_{\theta}(d)=1-\exp(-\theta d^k)\) for \(d\ge 0\).
We set \(k=1.5\), \(\theta^\star=0.5\), and \(T=600\), and service levels \(\gamma\in\{0.5,0.9,0.98\}\). 
We compare ChronosFlow-ICGPS to:
(a) TS--Weibull with conjugate updating for \(\theta^\star\) under a Gamma prior,
(b) a myopic plug-in MLE fit using uncensored observations only,
(c) an optimistic (UCB-style) estimator for the Weibull rate.
We report cumulative Bayesian regret relative to the clairvoyant Bayes-optimal fixed action oracle that orders the \(\gamma\)-quantile:
\(X^\star=(F^\star)^{-1}(\gamma)\).

In Figure~\ref{fig:exp1-weibull}, we find that ChronosFlow-ICGPS achieves comparable or slightly better regret than TS--Weibull across all \(\gamma\), indicating that completion-based posterior sampling recovers classical TS behavior when the parametric model is correctly specified.
In contrast, the myopic and optimistic baselines incur substantially higher regret when censoring is
frequent, especially at \(\gamma=0.5\), where the optimal policy stockouts often and thus produces
many censored observations. The gap shrinks as \(\gamma\) increases and observations become more
informative.

\begin{figure}[t]
    \centering
    \includegraphics[width=0.725\linewidth]{"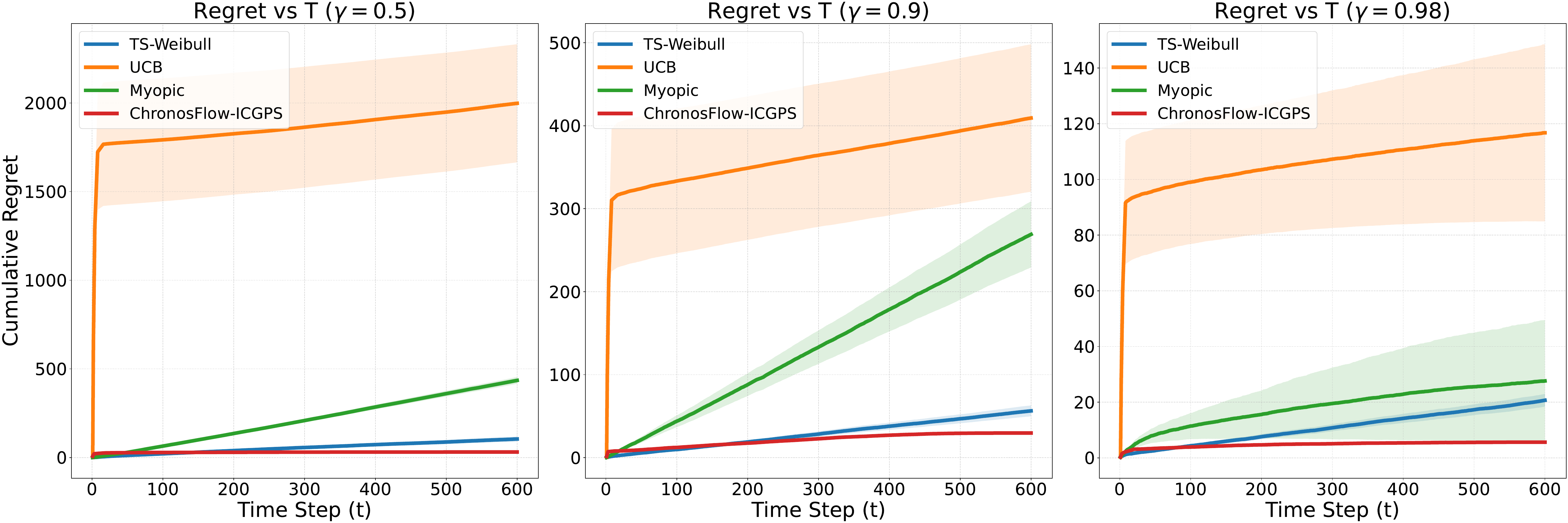"}
    \caption{Experiment~\ref{subsec:exp1}: cumulative Bayesian regret under correct specification for
    $\gamma\in\{0.5,0.9,0.98\}$. ChronosFlow-ICGPS matches the TS baseline across service levels,
    validating completion-based posterior sampling in the correctly specified i.i.d.\ regime.}
    \label{fig:exp1-weibull}
\end{figure}

\subsubsection{(Q2) Offline predictive quality and online regret}
\label{subsec:exp2}

In this experiment, we verify whether improving the offline predictive model reduces online regret in a fixed environment, as
suggested by Corollary \ref{cor:nv-final-regret-obs} in Section \ref{subsec:nv-theory}. We keep the online environment and the ChronosFlow-ICGPS online wrapper (Chronos backbone, KM module, and summary
mechanism) constant, and vary only the offline-trained completion kernel via two sweeps: (a) the number of offline training tasks and (b) the CNF head hidden size. We measure offline completion-model quality by a censoring-aware validation NLL
\(
\widehat{\cL}_{\mathrm{cens}}(\theta)
~:=~
-\frac{1}{|\cD_{\mathrm{val}}|}
\sum_{\tau\in\cD_{\mathrm{val}}}\;\sum_{t=1}^T
\log p_\theta\!\bigl(O_t^{(\tau)} \mid H_{t-1}^{(\tau)}\bigr),
\)
on held-out censored episodes $\cD_{\mathrm{val}}$. 

We plot the excess censoring-aware NLL $\widehat{\Delta}_{\mathrm{obs}}$ against the online terminal regret $\BReg_T$. Figure~\ref{fig:exp2-offline-vs-online} shows that models with better censoring-aware likelihood fit
tend to achieve lower online cumulative regret, most clearly in the data-scale row with lower censoring levels. The capacity sweep exhibits a weaker and noisier correspondence, including occasional mismatches, especially at higher censoring. This indicates that
architectural capacity alone does not monotonically translate into improved decision performance.

\begin{figure}[t]
    \centering
    \includegraphics[width=0.725\linewidth]{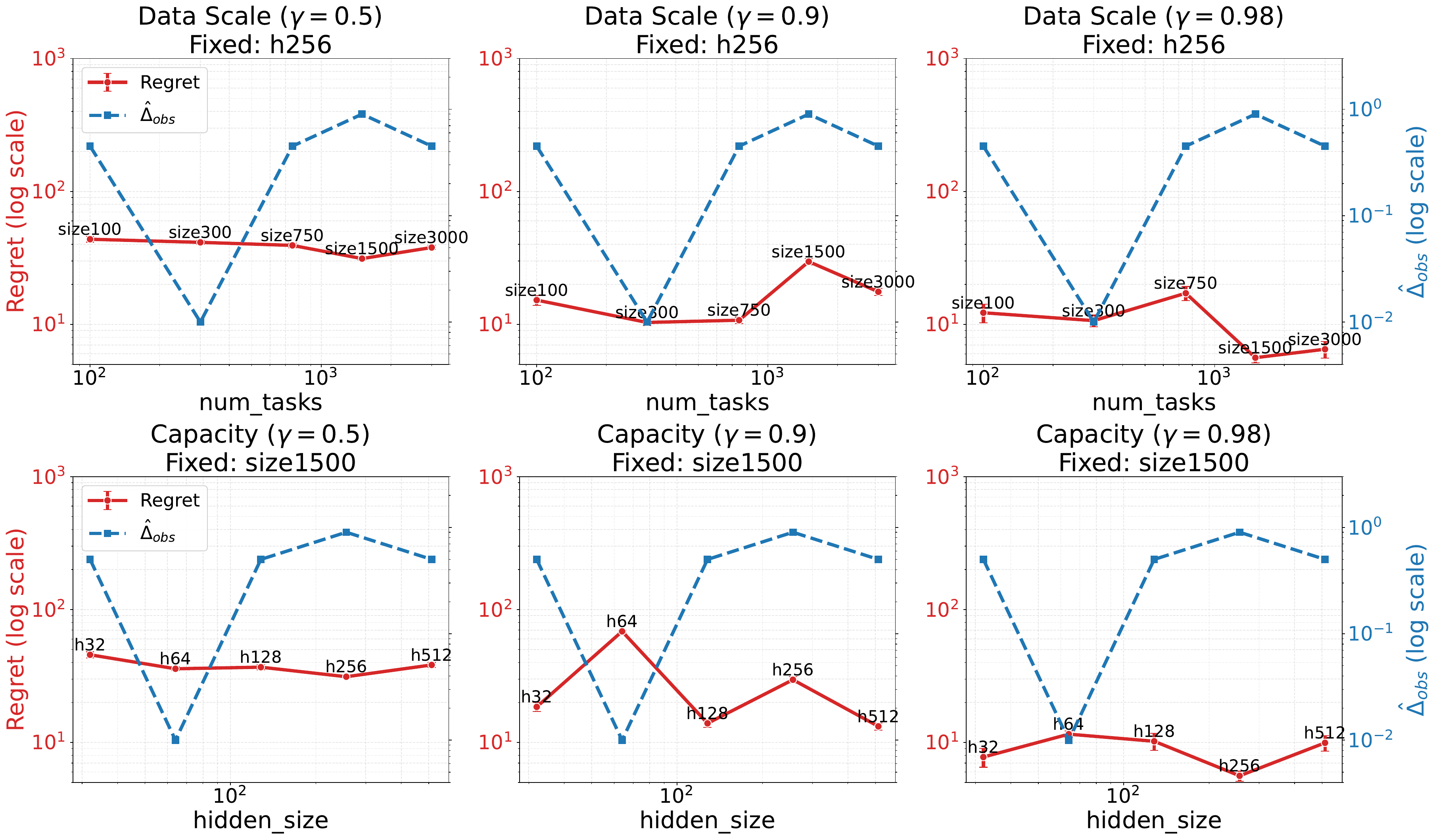}
    \caption{Experiment~\ref{subsec:exp2}: offline predictive quality versus online regret.
    Each point is an ICGPS model trained by varying either offline data scale (top row) or CNF head
    capacity (bottom row). Across sweeps, better censoring-aware validation NLL
    (smaller $\widehat{\Delta}_{\mathrm{obs}}$) generally correlates with lower online cumulative regret,
    with the relationship strongest for data scale ($\gamma = 0.9, 0.98$) and weaker (occasionally non-monotone) for capacity.}

    \label{fig:exp2-offline-vs-online}
\end{figure}

\subsubsection{(Q3) Robustness and transfer}
\label{subsec:exp3}
We evaluate the robustness of in-context GPS under \emph{distribution family shift} \emph{prior mismatch} within a family, that is endowed by meta-training. We use a short horizon (\(T=60\)) to emphasize cold-start performance, where
offline transfer is most impactful.

\paragraph{(a) Distribution shift}
\label{subsubsec:exp3-dist-shift}
We pretrain the completion kernel on a mixture of $K=4$ demand families (exponential, lognormal, Gompertz, log-logistic) and evaluate online on a held-out Weibull environment, holding the
total offline sample budget fixed. Baselines are (a) TS with correctly specified prior, respecting the Weibull test environment, and (b) misspecified TS aligned with the training mixture. Figure~\ref{fig:exp3-dist-shift} (shown for $K=4$) indicates that ChronosFlow-ICGPS substantially
reduces the regret incurred by misspecified TS under family shift, and for $\gamma\in\{0.9,0.98\}$
tracks the correctly specified TS closely. At $\gamma=0.5$, ChronosFlow exhibits a larger cold-start
offset but does not suffer the sustained growth of the misspecified baseline.

\begin{figure}[t]
    \centering
    \includegraphics[width=0.725\linewidth]{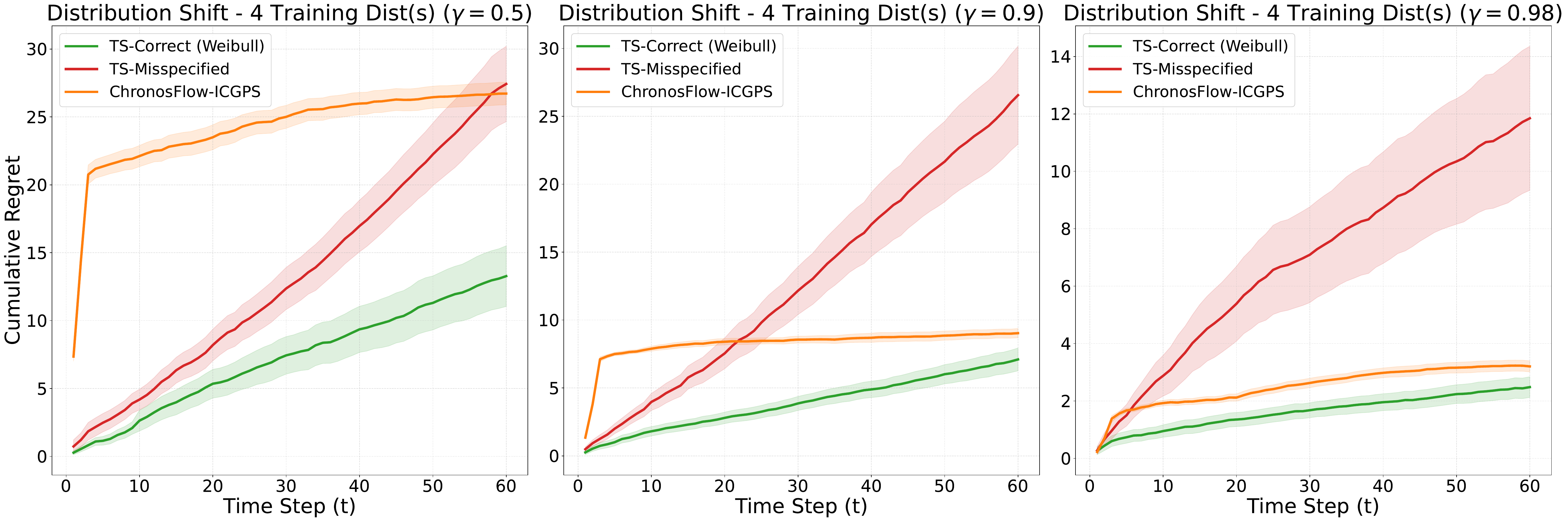}
    \caption{Experiment~\ref{subsec:exp3} (a): Distribution shift (test family Weibull) with $K=4$
    training families. ChronosFlow-ICGPS mitigates misspecification under family shift and is close to
    correctly specified TS for $\gamma\in\{0.9,0.98\}$; at $\gamma=0.5$ it shows a larger cold-start offset.
    }
    \label{fig:exp3-dist-shift}
\end{figure}

\emph{(b) Prior mismatch}
\label{subsubsec:exp3-prior-mismatch}
We also verify robustness to prior mismatch within the Weibull family by evaluating policies on an
online environment whose parameters lie outside the regime seen during offline training. Offline training tasks cover a demand regime with $\lambda\in[0.2,0.4]$ and $k\in[0.8,1.2]$, while the online test environment uses $\lambda^\star=0.8$ and $k=2.0$. We compare ChronosFlow-ICGPS trained on this offline corpus to \texttt{TS-correct} (oracle TS prior aligned to test environment), and \texttt{TS-train-prior} (TS using mismatched prior induced by offline corpus). Figure~\ref{fig:exp3-prior-mismatch} shows that a misaligned prior can severely degrade parametric TS,
while ChronosFlow-ICGPS remains stable and substantially reduces regret by adapting from censored
online feedback.
\begin{figure}[t]
    \centering
    \includegraphics[width=0.725\linewidth]{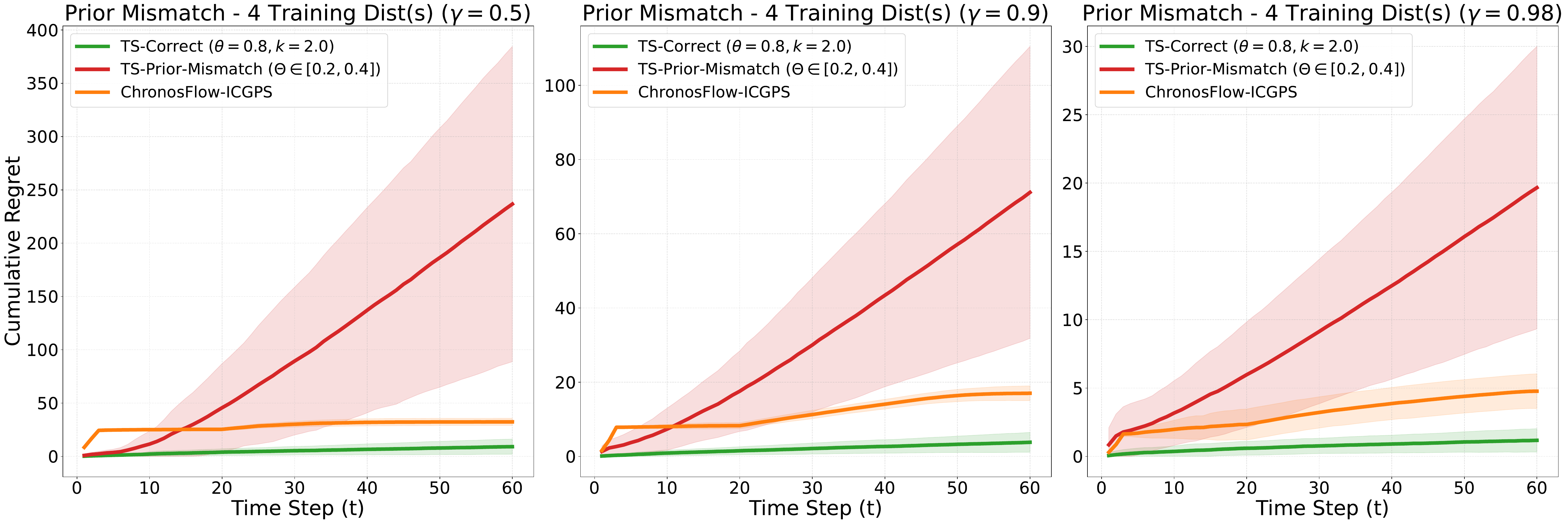}
    \caption{Experiment~\ref{subsubsec:exp3-prior-mismatch} (b): prior mismatch robustness for
    $\gamma\in\{0.5,0.9,0.98\}$. Misaligned priors can induce large TS regret, while ChronosFlow-ICGPS remains stable via meta-training.}

    \label{fig:exp3-prior-mismatch}
\end{figure}

\subsubsection{(Q4) Real-world dataset}
\label{subsec:exp4}
For real-world data experiments, we follow the setup (preprocessing, season-based splits, and evaluation
protocol) of \citet[Sec.~6.3]{hssaine2024censored} and benchmark ChronosFlow-ICGPS on the SuperStore dataset \citep{Sahoo2023SuperstoreSales}, used in their paper. Each (product, season/store) defines an episode with right-censored
feedback induced by stockouts, constructed from time-ordered sales/stockout records
under their selling-season rules. We sweep the censoring-control parameter \(\lambda\) (smaller
\(\lambda\) implies heavier censoring), meta-train on historical seasons, and evaluate online on held-out
seasons. We report two variants: \texttt{Native} (completion kernel trained only on the target dataset)
and \texttt{Meta} (completion kernel pretrained on the other real datasets; deployed under the same
online protocol). 

The dataset consists of three splits: \textsc{Technology}, \textsc{Office Supplies}, and \textsc{Furniture}, out of which we report the first here in Table \ref{tab:exp4-real_1}, and defer the rest of the results (Tables~\ref{tab:exp4-real_2}--\ref{tab:exp4-real_3}) as well as full construction details to Appendix~\ref{app:exp4}. Across categories, ChronosFlow-ICGPS is competitive and delivers its clearest gains under severe
censoring (\(\lambda\le 3\)); for example, on \textsc{Furniture} at \(\lambda=1\), \texttt{Meta} achieves
\(4.1\) (best), and on \textsc{Technology} at \(\lambda=3\), \texttt{Meta} attains
\(4.2\) (best). As censoring relaxes, gaps typically narrow in a category-dependent manner; \texttt{Meta} remains comparatively stable across \(\lambda\), while \texttt{Native} can degrade at light censoring. Overall, the robustness benefits observed in synthetic experiments carry over
to real, non-i.i.d.\ demand with different levels of censoring.

\section{Conclusion}
This work studies sequential inventory control under \emph{decision-dependent right censoring} in the repeated newsvendor (R-NV): the order quantity determines what is observed, and stockouts censor demand so the learner observes only sales. To overcome the limitations of existing approaches, including choice of priors and inability to transfer to online settings, we propose the in-context generative posterior sampling methodology, which consists of offline meta-training and online in-context sampling, both respecting the censored feedback structure. 
On the theoretical front, our analysis provides a regret decomposition consisting of (i) the regret of an ideal TS benchmark that uses the true completion kernel, and (ii) a deployment penalty bounded by a completion-mismatch quantity. We show that in the R-NV setting, the TS benchmark admits sublinear Bayesian regret via a reduction to a derived bandit convex optimization (BCO) feedback, while the completion-mismatch term links online performance to offline predictive fit of the completion model. Decision-dependent censoring can limit information about tail demand and create identifiability challenges in the offline-to-online transfer, without additional structural assumptions, which we derive precisely for ICGPS applied to R-NV problem. Empirically, the ChronosFlow-ICGPS architecture is robust to model and prior misspecification under censoring. On the SuperStore real-world data benchmark under heavy censoring, ChronosFlow-ICGPS achieves strong performance, especially for the \textsc{Meta} algorithm variant. A natural extension is a \emph{contextual} version of R-NV, to include features like weather, price, text, etc., along the lines of \cite{hssaine2024censored, zhang2025contextual}. A second direction is tackling \emph{non-stationarity} in demands \cite{DBLP:journals/ior/BesbesGZ15} using recent advances in non-stationary BCO \cite{JMLR:v22:20-763}.

\section*{AI Usage}
We used an LLM (GPT v5.2) as assistive tool for: (i) improving exposition via language editing; (ii) aiding understanding of theoretical concepts; and (iii) code review. All LLM-generated suggestions were independently verified by at least one (and in most cases multiple) co-authors.

\begin{table*}[t]
\centering
\caption{Results of Experiment \ref{subsec:exp4} for the \textsc{Technology} Dataset}
\footnotesize
\setlength{\tabcolsep}{3pt}
\renewcommand{\arraystretch}{1.05}
\begin{adjustbox}{max width=\textwidth,center}
\begin{tabular}{l*{15}{c}}
\toprule
Algorithm & $\lambda$=1 & 2 & 3 & 4 & 5 & 6 & 7 & 8 & 9 & 10 & 11 & 12 & 13 & 14 & 15 \\
\midrule
ChronosFlow-ICGPS (\texttt{Native}) & 12.2 & \textbf{4.5} & 5.1 & 6.2 & 5.1 & 6.0 & 5.4 & 4.7 & 5.1 & 4.2 & 5.0 & 6.0 & 6.3 & 8.0 & 6.8 \\
ChronosFlow-ICGPS (\texttt{Meta}) & \textbf{11.5} & 6.7 & \textbf{4.2} & \textbf{3.9} & 4.0 & \textbf{3.9} & \textbf{3.9} & \textbf{3.9} & \textbf{3.9} & \textbf{3.9} & \textbf{3.9} & \textbf{3.9} & 3.9 & 3.9 & \textbf{3.9} \\
SAA & 15.1 & 15.1 & 9.5 & 6.2 & 4.4 & 4.4 & \textbf{3.9} & \textbf{3.9} & \textbf{3.9} & \textbf{3.9} & \textbf{3.9} & \textbf{3.9} & \textbf{3.9} & \textbf{3.9} & \textbf{3.9} \\
RCN \citep{hssaine2024censored} & 19.9 & 18.9 & 17.1 & 14.7 & \textbf{3.9} & 4.1 & \textbf{3.9} & \textbf{3.9} & \textbf{3.9} & \textbf{3.9} & \textbf{3.9} & \textbf{3.9} & \textbf{3.9} & \textbf{3.9} & \textbf{3.9} \\
Kaplan-Meier & 15.1 & 9.5 & 6.2 & 4.4 & \textbf{3.9} & \textbf{3.9} & \textbf{3.9} & \textbf{3.9} & \textbf{3.9} & \textbf{3.9} & \textbf{3.9} & \textbf{3.9} & \textbf{3.9} & \textbf{3.9} & \textbf{3.9} \\
\bottomrule
\end{tabular}
\end{adjustbox}
\label{tab:exp4-real_1}
\end{table*}

%% file: sec-apx_method.tex
\newcommand{\Ind}{\mathbf{1}}

\section{Methodological Details}

\begin{algorithm}[h]
\caption{In-context Generative Posterior Sampling (GPS) for right-censored newsvendor}
\label{alg:nv-gps}
\small
\begin{algorithmic}[1]
\Require critical fractile $\gamma=b/(b+h)$; learned completion model $p_\theta(d_{1:T}\mid h)$
\State $H_0 \gets \emptyset$
\For{$t \gets 1,2,\ldots,T$}
  \State \textsc{Completion sampling} sample $\widetilde D_{1:T}^{(t)} \sim p_\theta(\cdot\mid H_{t-1})$ subject to constraints in~\eqref{eq:consistency-constraints}
  \State Initialize $\widetilde D_{1:T}^{(t)}$ as an empty vector
  \For{$s \gets 1,2,\ldots,T$}
    \If{$s<t$ \textbf{and} $C_s=1$}
      \State $\widetilde D_s^{(t)} \gets S_s$ \Comment{No stockout: demand is revealed}
    \ElsIf{$s<t$ \textbf{and} $C_s=0$}
      \State Sample $\widetilde D_s^{(t)}$ from $p_\theta(\cdot\mid \widetilde D_{1:s-1}^{(t)},H_{t-1})$
      conditioned on $\{\widetilde D_s^{(t)} > X_s\}$ via \eqref{eq:tail-density}
      \Comment{Lemma~\ref{lem:tail-conditioning}}
    \Else
      \State Sample $\widetilde D_s^{(t)} \sim p_\theta(\cdot\mid \widetilde D_{1:s-1}^{(t)},H_{t-1})$
      \Comment{Future demand is unconstrained}
    \EndIf
  \EndFor
  \State \textsc{Decision} $X_t \gets \hat x(\widetilde D_{1:T}^{(t)})$
  \Comment{Left empirical $\gamma$-quantile Sec.~\ref{sec:nv_running_example}}
  \State \textsc{Feedback} Observe $O_t=(S_t,C_t)=\psi(X_t,D_t)$ and set $H_t\gets(H_{t-1},(X_t,O_t))$
\EndFor
\end{algorithmic}
\end{algorithm}

\begin{table*}[t]
\caption{Examples of operational problems that fit the censored feedback template. Note that the latent outcome $U_t$ is a minimal sufficient latent variable, and richer outcomes can be used when convenient.}
\label{tab:app_examples}
\centering
\footnotesize
\setlength{\tabcolsep}{4pt}
\renewcommand{\arraystretch}{1.15}

\begin{tabularx}{\linewidth}{@{}
  >{\raggedright\arraybackslash}p{0.22\linewidth}
  >{\raggedright\arraybackslash}p{0.13\linewidth}
  >{\raggedright\arraybackslash}p{0.14\linewidth}
  >{\raggedright\arraybackslash}X
  >{\raggedright\arraybackslash}p{0.18\linewidth}
@{}}
\toprule
\textbf{Instance (key refs)} &
\textbf{Action $X_t$} &
\textbf{Latent outcome $U_t$} &
\textbf{Observation $O_t=\psi(X_t,U_t)$} &
\textbf{Loss $\ell(X_t,U_t)$} \\
\midrule

Censored newsvendor \cite{huh2011adaptive,besbes2022tradeoff} &
Order quantity &
Demand &
Sales and stockout indicator &
Newsvendor cost \\

\addlinespace
Capacity controls in revenue management \cite{TalluriVanRyzin2004,littlewood2005forecasting} &
Booking limit &
Requests without limit &
Accepted bookings &
Negative revenue \\

\addlinespace
Budget pacing in digital advertising \cite{BalseiroEtAl2023RobustPacing,ConitzerEtAl2022PacingEq} &
Spend cap &
Uncapped spend &
Realized spend &
Penalty for underspend/overspend \\

\addlinespace
Posted-price mechanisms \cite{KleinbergLeighton2003,Myerson1981} &
Posted price &
Buyer valuation &
Purchase indicator &
Negative revenue \\

\bottomrule
\end{tabularx}
\end{table*}

\subsection{General template for operational problems with decision-dependent uncertainty}
\label{app:general_template}

This appendix presents a generic template for sequential decision problems with scalar actions used implicitly throughout the main text, and provides a mapping for several operational examples.
Let $X\subset \mathbb{R}$ be a compact interval (the action space).
A policy $\pi$ sequentially chooses actions $X_t\in X$ for $t=1,\dots,T$ based on past observations.
We assume there exists an underlying outcome space $\mathcal{U}$ and a random trajectory
$U_{1:T}:=(U_1,\dots,U_T)\in\mathcal{U}^T$ such that the per-round loss is a known measurable function
$\ell: X\times\mathcal{U}\to[0,1]$, and the learner observes a (possibly censored/coarsened) feedback signal
\(
O_t = \psi(X_t, U_t)
\)
for a known measurable map $\psi: X\times\mathcal{U}\to\mathcal{O}$.
The history at time $t$ is
\(
H_t := \bigl( (X_s,O_s)\bigr)_{s=1}^t \in (X\times\mathcal{O})^t,  H_0:=\emptyset,
\)
and a (possibly randomized) policy maps $H_{t-1}$ to a distribution over $X$.
Define the Bayes risk
\(
f^\star(x) := \mathbb{E}_{P^\star}\!\left[\ell(x,U_t)\right],
\)
where in the stationary case the RHS does not depend on $t$.
The unknown environment is captured by a distribution $\P^\star$ over complete trajectories $U_{1:T}$. Table \ref{tab:app_examples} provides additional examples of operational problems that admit decision-dependent uncertainty due to censored feedback, and can therefore possibly inherit, with minor modifications, the method and theory presented in this work for the repeated newsvendor example.

\subsection{Function-valued uncertainty in continuous action problems}
\label{subsec:continuous}
For problems with finite action sets, posterior sampling is typically described as sampling a latent parameter or a finite vector of missing potential outcomes, and then choosing the best action under that sampled instance. When $\cX$ is a continuum, it is typically useful to view the unknown object as an
\emph{entire loss (or risk) function} over $\cX$ and to sample this function from its posterior.

\paragraph{Function-valued latent object.}
Let us assume that the Bayes risk $f^\star:\cX\to[0,1]$ is convex and $L$-Lipschitz on the
compact interval $\cX\subset\R$, hence continuous. It is therefore natural to regard $f^\star$ as a random
element of the Polish space $(C(\cX),\|\cdot\|_\infty)$ equipped with its Borel $\sigma$-algebra.
In a Bayesian model, randomness in $f^\star$ arises from randomness in the underlying environment
(e.g., an unknown demand model), and conditioning on the observed history $H_{t-1}$ induces a posterior
distribution $\Law(f^\star\mid H_{t-1})$ over functions.

\paragraph{Posterior sampling in function space.}
Fix any measurable tie-breaking rule that selects a minimizer from a nonempty set
(e.g., the smallest minimizer).\footnote{Measurable selection issues can be bypassed by a fixed
tie-breaking rule, or by working with a vanishing strongly convex regularizer so that the minimizer is unique;
both options are standard in continuous-action analyses.}
Then posterior sampling can be written as
\begin{equation}
\label{eq:ts-function}
\widetilde f^{(t)} \sim \Law(f^\star\mid H_{t-1}),
\qquad
X_t \in \argmin_{x\in\cX}\widetilde f^{(t)}(x).
\end{equation}
This is essentially Thompson sampling in \emph{bandit convex optimization (BCO)}: the learner is optimizing an unknown convex function over a continuum using only bandit feedback. In 1D, Thompson sampling enjoys $\widetilde O(\sqrt{T})$ Bayesian regret in BCO under
convexity, boundedness, and a mild Lipschitz condition \citep{bakhtiari2025bco} .

\paragraph{Equivalence of function and outcome sampling.}
The missing-data formulation of Section~\ref{subsec:missingdata} samples a completion
$\widetilde U_{1:T}^{(t)}\sim \P^\star(\cdot\mid H_{t-1})$ and then applies an oracle map, e.g., the ERM oracle. The function-valued viewpoint is simply the pushforward of this posterior through the map that associates a completion to a risk function. To make this explicit, define the (random) empirical loss function associated with a complete trajectory
$u_{1:T}\in\cU^T$: \(\widehat f_{u_{1:T}}(x) := \frac1T\sum_{s=1}^T \ell(x,u_s),\) where \(x\in\cX. \)
Let $\Phi:\cU^T\to C(\cX)$ denote the measurable map $\Phi(u_{1:T})=\widehat f_{u_{1:T}}(\cdot)$.
Then the posterior distribution of $\widehat f_{U_{1:T}}$ given history is the pushforward measure
$\Law(\widehat f_{U_{1:T}}\mid H_{t-1}) = \Law(\Phi(U_{1:T})\mid H_{t-1})$.

\begin{lemma}[Outcome-sampling and function-sampling induce the same action law]
\label{lem:outcome-vs-function}
Let $\hat x(u_{1:T}) \in \argmin_{x\in\cX}\widehat f_{u_{1:T}}(x)$ be an oracle action computed from a completed
trajectory, using a fixed measurable tie-break rule.
Let $\widetilde U_{1:T}^{(t)} \sim \P^\star(\cdot\mid H_{t-1})$ and set
$X_t := \hat x(\widetilde U_{1:T}^{(t)})$.
Alternatively, sample $\widetilde f^{(t)}\sim \Law(\widehat f_{U_{1:T}}\mid H_{t-1})$ and set
$\widetilde X_t\in \argmin_{x\in\cX}\widetilde f^{(t)}(x)$ (with the same tie-break rule).
Then
\(
\Law(X_t\mid H_{t-1})=\Law(\widetilde X_t\mid H_{t-1}).
\)
\end{lemma}

\begin{proof}
By definition, $\widetilde f^{(t)}=\Phi(\widetilde U_{1:T}^{(t)})$ in distribution conditional on $H_{t-1}$.
Since $\hat x(\cdot)$ is a measurable function of $\Phi(\cdot)$ (it depends on $u_{1:T}$ only through
$\widehat f_{u_{1:T}}$ and the tie-break rule), we have
\(
X_t=\hat x(\widetilde U_{1:T}^{(t)}) = \argmin_{x\in\cX}\Phi(\widetilde U_{1:T}^{(t)})(x).
\)
Thus $X_t$ is the pushforward of $\widetilde U_{1:T}^{(t)}$ under the measurable map
$u_{1:T}\mapsto \argmin_x \Phi(u_{1:T})(x)$, which is exactly the same pushforward used to define $\widetilde X_t$.
\end{proof}

Lemma~\ref{lem:outcome-vs-function} shows that \emph{sampling missing outcomes and optimizing} and
\emph{sampling a function and optimizing} are mathematically the same operation: the latter is just the
pushforward of the former through a map from completed data to a loss function.
This equivalence is the continuous-action analogue of probability matching
(Lemma~\ref{lem:probmatching}).

\begin{remark}[Non-stationary environments]
\label{rem:nonstationary}
The function-valued viewpoint does \emph{not} inherently require stationarity; what stationarity buys is a \emph{time-invariant comparator} and a \emph{single} latent risk function. If the data-generating process is non-stationary, the relevant object becomes a sequence
$f^\star_1,f^\star_2,\dots$, where, for example, $f_t^\star(x):=\Eb{\ell(x,U_t)}$ may vary with $t$. One then typically studies \emph{dynamic regret} (tracking a changing benchmark) or imposes structure such as piecewise stationarity or bounded variation. Notably, even in the fully adversarial 1D BCO model---where an arbitrary sequence of convex loss functions is chosen---the minimax regret remains $\widetilde\Theta(\sqrt{T})$; the proof reduces the game to a Bayesian formulation and solves it
using a variant of Thompson sampling \citep{bubeck2015bco1d}. This indicates that stationarity is not necessary for
sublinear regret \emph{in principle}. In our paper, we adopt the stationary formulation because it matches the operational setting (a fixed environment generating repeated instances). If drift is expected, a natural extension is to encode time or seasonality as context \citealp{zhang2025contextual} or to use windowed conditioning in the generative model.
\end{remark}

%% file: sec-apx_theory.tex
\begin{table*}[t]
\ContinuedFloat
\centering
\small
\setlength{\tabcolsep}{6pt}
\renewcommand{\arraystretch}{1.15}
\begin{tabular}{p{0.26\linewidth} p{0.70\linewidth}}
\toprule
\textbf{Symbol} & \textbf{Meaning}\\

\midrule
\multicolumn{2}{l}{\textit{Histories and trajectories}}\\
$Z_t:=(X_t,O_t)$ & Interaction tuple at time $t$.\\
$H_t:=(Z_1,\dots,Z_t)$ & History up to time $t$; $H_0:=\varnothing$.\\
$X_{1:t},O_{1:t},D_{1:T}$ & Shorthand for sequences, e.g.\ $X_{1:t}=(X_1,\dots,X_t)$.\\

\midrule
\multicolumn{2}{l}{\textit{Policies and induced interaction law}}\\
$\pi=(\pi_t)_{t=1}^T$ & Policy; each $\pi_t(\cdot\mid h_{t-1})$ is a distribution over actions in $\mathcal{X}$.\\
$\P^\star$ & True environment governing latent demands (and thus observations via $\psi$).\\
$P_{\pi_\star}$ & Law of the full interaction history $H_T$ under environment $\P^\star$ and policy $\pi$.\\

\midrule
\multicolumn{2}{l}{\textit{Completion kernels}}\\
$q_t^\star(\cdot\mid h)$ & True one-step completion kernel: $\mathrm{Law}( D_t\mid H_{t-1}=h)$ on $\mathcal{D}=[0,B]$.\\
$q_{t,\theta}(\cdot\mid h)$ & Learned one-step completion kernel used online: $\mathrm{Law}_\theta(\widetilde D_t^{(t)}\mid H_{t-1}=h)$.\\
$Q_t^\star(\cdot\mid h)$ & True trajectory completion kernel: $\mathrm{Law}( D_{1:T}\mid H_{t-1}=h)$ on $\mathcal{D}^T$.\\
$Q_{t,\theta}(\cdot\mid h)$ & Learned trajectory completion kernel used online: $\mathrm{Law}_\theta(\widetilde D_{1:T}^{(t)}\mid H_{t-1}=h)$.\\

\midrule
\multicolumn{2}{l}{\textit{Induced action kernels}}\\
$K_t^\star(\cdot\mid h)$ & Induced (oracle) action kernel: $K_t^\star=(\varphi_t(\cdot,h))_{\#}Q_t^\star(\cdot\mid h)$.\\
$K_{t,\theta}(\cdot\mid h)$ & Learned induced action kernel: $K_{t,\theta}=(\varphi_t(\cdot,h))_{\#}Q_{t,\theta}(\cdot\mid h)$.\\

\midrule
\multicolumn{2}{l}{\textit{Observed kernels}}\\
$P^{\mathrm{obs}}_{h,x}$ & True conditional law of $O_t$ given $(H_{t-1},X_t)=(h,x)$.\\
$Q^{\mathrm{obs}}_{h,x}$ & Model conditional law of $O_t$ given $(H_{t-1},X_t)=(h,x)$ under parameter $\theta$.\\

\midrule
\multicolumn{2}{l}{\textit{Mismatch functionals}}\\
$\Delta^{\mathrm{comp}}_{T}(\theta;\pi_\theta)$ &
$\sum_{t=1}^T \mathbb{E}_{H_{t-1}\sim \P_{\pi_\theta}^\star}\!\big[\mathrm{KL}(Q_t^\star(\cdot\mid H_{t-1})\|Q_{t,\theta}(\cdot\mid H_{t-1}))\big]$.\\
$\Delta_{T}^{\mathrm{act}}(\theta)$ &
$\sum_{t=1}^T \mathbb{E}_{H_{t-1}\sim \P_{\pi_\theta}^\star}\!\big[\mathrm{KL}(K_{t,\theta}(\cdot\mid H_{t-1})\|K_t^\star(\cdot\mid H_{t-1}))\big]$.\\
$\Delta^{\mathrm{obs}}_{T}(\theta;\pi_\theta)$ &
$\sum_{t=1}^T \mathbb{E}_{(H_{t-1},X_t)\sim \P_{\pi_\theta}^\star}\!\big[\mathrm{KL}(P^{\mathrm{obs}}_{H_{t-1},X_t}\|Q^{\mathrm{obs}}_{H_{t-1},X_t})\big]$.\\
\bottomrule

\end{tabular}
\caption[]{Summary of notations used in proofs.}
\end{table*}

\section{Notations}
\label{sec:notation}

\paragraph{Censored NV  problem.}
We consider a horizon $T\in\mathbb{N}$ with rounds indexed by $t\in[T]:=\{1,\dots,T\}$. Let the \emph{action} space be $\mathcal{X}:=[0,B]$ and the \emph{demand} space be $\mathcal{D}:=[0,B]$.
The latent demand is $D_t\in\mathcal{D}$. The decision (order quantity) is $X_t\in\mathcal{X}$. Given $(X_t,D_t)$, the observed feedback is
\(
S_t:=\min\{D_t,X_t\}\in[0,B],
C_t:=\mathbf{1}\{D_t\le X_t\}\in\{0,1\},
O_t:=(S_t,C_t)\in\mathcal{O}:=[0,B]\times\{0,1\}.
\)
Equivalently, $O_t=\psi(X_t,D_t)$ where $\psi(x,d)=(\min\{d,x\},\mathbf{1}\{d\le x\})$.
We reserve the letter $C$ exclusively for the censoring indicator (no other object uses $C$). Define $Z_t:=(X_t,O_t)$ and the history (interaction trajectory) up to time $t$ as
\(H_t:=(Z_1,\dots,Z_t) = \bigl((X_s,O_s)\bigr)_{s=1}^t, H_0:=\varnothing.
\)
We write realizations as $h_t=(z_1,\dots,z_t)$, and use the shorthand $X_{1:t}=(X_1,\dots,X_t)$, $O_{1:t}=(O_1,\dots,O_t)$, etc.

\paragraph{Policies and induced laws.}
For each \(t\), the environment is specified by a Markov kernel
\(M^\star_t(\cdot \mid h_{t-1},x_t)\in\Delta(\mathcal{O}).\)
A (possibly randomized, history-dependent) policy is $\pi=(\pi_t)_{t=1}^T$ where each
$\pi_t(\cdot\mid h_{t-1})$ is a probability measure on $\mathcal{X}$.
Let $\P^\star$ denote the true environment (Bayesian model) governing latent demands.
Let $P_{\pi}^\star$ denote the probability law of the entire interaction history $H_T$
generated by $\P^\star$ when actions are chosen according to $\pi$ and observations are produced
via the censoring map $\psi$. The policy \(\pi\) is given by action kernels
\(
K^{\pi}_t(\cdot \mid h_{t-1})\in\Delta(\mathcal{X}), t=1,\dots,T,
\)
and induces a trajectory law \(\P^\star_{\pi}\) on \(Z_{1:T}=(X_1,O_1,\dots,X_T,O_T)\) through
\(
X_t \mid H_{t-1}\sim K^{\pi}_t(\cdot\mid H_{t-1})\), and
\(O_t \mid (H_{t-1},X_t)\sim M^\star_t(\cdot\mid H_{t-1},X_t).
\)

\paragraph{Completion kernels (trajectory-level conditional laws).}
At time $t$, given a realized history $h_{t-1}$, define the \emph{true completion kernel}
as the conditional law of a full latent trajectory:
$Q_t^\star(\,\cdot \mid h_{t-1}) := \mathrm{Law}\bigl(D_{1:T}\mid H_{t-1}=h_{t-1}\bigr), \text{a probability measure on }\mathcal{D}^T$. Let $p_\theta(\cdot\mid h_{t-1})$ be the learned generative model used to sample completions. Because observed right-censoring imposes hard feasibility constraints (e.g.\ if $C_s=1$ then $D_s=S_s$, and if $C_s=0$ then $D_s>X_s$ for $s<t$), we define the
\emph{learned completion kernel} as the conditional law actually used online after enforcing
these constraints: \(Q_{t,\theta}(\,\cdot \mid h_{t-1}) := \mathrm{Law}_\theta\bigl(\widetilde D_{1:T}^{(t)} \mid H_{t-1}=h_{t-1}\bigr),\)
where $\widetilde D_{1:T}^{(t)}\in\mathcal{D}^T$ denotes the completed trajectory sampled at time $t$. The \emph{true one-step completion kernel} is the conditional law of the next demand:
\(q_t^\star(\,\cdot \mid h_{t-1})
:= \mathrm{Law}\!\left( D_t \,\middle|\, H_{t-1}=h_{t-1}\right), \text{a probability measure on }\mathcal{D}=[0,B]\). The \emph{learned one-step completion kernel} is
\(q_{t,\theta}(\,\cdot \mid h_{t-1})
:= \mathrm{Law}_{\theta}\!\left(\widetilde D_t^{(t)} \,\middle|\, H_{t-1}=h_{t-1}\right)\), where $\widetilde D_t^{(t)}$ is the time-$t$ coordinate of the completion sampled online at round $t$.

\paragraph{Induced action kernels and observed kernels.}
Let $\varphi_t:\mathcal{D}^T\times \mathcal{H}_{t-1}\to\mathcal{X}$ be a measurable oracle map
that converts a completion into an action (possibly depending on the history).
In the repeated newsvendor, the canonical choice is empirical-risk minimization on the completion:
\(
\hat x(D_{1:T}) \in \arg\min_{x\in[0,B]} \frac{1}{T}\sum_{s=1}^T \ell(x,D_s), \text{and we can take }\varphi_t(D_{1:T},h_{t-1})=\hat x(D_{1:T}).
\)
Define the induced \emph{action kernels} (conditional laws of $X_t$ given history) via pushforward:
\(
K_t^\star(\,\cdot\mid h_{t-1}) := (\varphi_t(\cdot,h_{t-1}))_{\#} Q_t^\star(\,\cdot\mid h_{t-1}),
K_{t,\theta}(\,\cdot\mid h_{t-1}) := (\varphi_t(\cdot,h_{t-1}))_{\#} Q_{t,\theta}(\,\cdot\mid h_{t-1}),
\) where the pushforward is defined by $(f_{\#}\mu)(A)=\mu(f^{-1}(A))$ for measurable $A$. Moreover, let $P^{\mathrm{obs}}_{h,x}$ and $Q^{\mathrm{obs}}_{h,x}$ denote the conditional laws of $O_t$
given $(H_{t-1},X_t)=(h,x)$ under the true environment and the learned model, respectively.

\paragraph{Mismatch functionals.}We use three KL-based mismatch functionals that are defined as follows: 

\begin{enumerate}
\item On-policy completion mismatch:
\[
\Delta^{\mathrm{comp}}_{T}(\theta;\pi_\theta)
:=\sum_{t=1}^T \mathbb{E}_{H_{t-1}\sim \P_{\pi_\theta}^\star}\!\left[
\mathrm{KL}\!\left(Q_t^\star(\cdot\mid H_{t-1}) \,\big\|\, Q_{t,\theta}(\cdot\mid H_{t-1})\right)\right].
\]
\item Causal action mismatch:
\[
\Delta_{T}^{\mathrm{act}}(\theta)
:=\sum_{t=1}^T \mathbb{E}_{H_{t-1}\sim \P_{\pi_\theta}^\star}\!\left[
\mathrm{KL}\!\left(K_{t,\theta}(\cdot\mid H_{t-1}) \,\big\|\, K_t^\star(\cdot\mid H_{t-1})\right)\right].
\]
\item On-policy observed mismatch:
\[
\Delta^{\mathrm{obs}}_{T}(\theta;\pi_\theta)
:=\sum_{t=1}^T \mathbb{E}_{(H_{t-1},X_t)\sim \P_{\pi_\theta}^\star}\!\left[
\mathrm{KL}\!\left(P^{\mathrm{obs}}_{H_{t-1},X_t}\,\big\|\, Q^{\mathrm{obs}}_{H_{t-1},X_t}\right)\right].
\]
\end{enumerate}

\section{Theoretical Proofs}

\subsection{Standard Lemmas}

\begin{lemma}[Donsker--Varadhan change-of-measure {\cite[Thm.~4.6]{PolyanskiyWu2025}}]
\label{lem:DV}
Let $P\ll Q$ be probability measures on $(\Omega,\mathcal F)$ and let $f:\Omega\to\mathbb{R}$ be measurable with
$\mathbb E_Q[e^{f}]<\infty$. Then,
$\mathrm{KL}(P\|Q) = \sup_f \{\mathbb E_P[f]-\log\mathbb E_Q[e^f]\}$. Equivalently, we have the following form
\[
\mathbb E_P[f] \;\le\; \mathrm{KL}(P\|Q) + \log \mathbb E_Q[e^{f}].
\]
\end{lemma}

\begin{lemma}[Conditional Hoeffding's lemma {\cite[Defn.~4.15]{PolyanskiyWu2025}}]
\label{lem:Hoeffding}
Let $\mathcal G\subseteq\mathcal F$ be a sub-$\sigma$-algebra and let $X$ be a real-valued
random variable such that $X\in[a,b]$ $Q$-almost surely for some $a<b$.
Then for all $\lambda\in\R$,
\[
\E_Q\!\left[\exp\!\Big(\lambda\big(X-\E_Q[X\mid\mathcal G]\big)\Big)\,\Big|\,\mathcal G\right]
\;\le\;
\exp\!\left(\frac{\lambda^2(b-a)^2}{8}\right)
\qquad Q\text{-a.s.}
\]
\end{lemma}

\begin{lemma}[MGF bound for adapted bounded martingale differences {\cite[Cor.~2.20]{wainwright2019hds}}]
\label{lem:md_mgf}
Let $(d_t)_{t=1}^T$ be an adapted sequence such that for each $t$,
\(
\E_Q[d_t\mid \mathcal F_{t-1}] = 0, Q\text{-a.s.}
\)
Assume there exists a deterministic $c\ge 0$ such that \(
d_t \in [A_t,B_t], Q\text{-a.s.}\), and \(B_t-A_t \le c, Q\text{-a.s.}
\)
Then for all $\lambda\in\R$, we have the following
\[
\log \E_Q\!\left[\exp\!\Big(\lambda\sum_{t=1}^T d_t\Big)\right]
\le
\frac{\lambda^2 T c^2}{8}.
\]
\end{lemma}

\begin{lemma}[Hellinger distance controlled by KL {\cite[Lemma A.5]{foster2021statistical}}]
\label{lem:hellinger-le-kl}
For any probability measures $P,Q$ with $P\ll Q$,
\[
D_H^2(P,Q) \le \mathrm{KL}(P\|Q).
\]
\end{lemma}

\begin{lemma}[KL controlled by Hellinger distance {\cite[Lemma A.10]{foster2021statistical}}]
\label{lem:kl-le-hellinger}
Let $P,Q$ be probability measures on a measurable space $(\mathcal X,\mathcal F)$. If, for some $V\in[1,\infty)$, $\sup_{F\in\mathcal F}\frac{P(F)}{Q(F)} \le V$, then we have the following
\[
\mathrm{KL}(P\|Q)\le (2+\log V)\,D_H^2(P,Q).
\]
\end{lemma}

\begin{lemma}[Data processing for KL {\cite[Thm.~7.4]{PolyanskiyWu2025}}]]
\label{lem:dp-kl}
Let $\psi:(\Omega,\mathcal F)\to(\Omega',\mathcal F')$ be measurable and let $P\ll Q$ on $(\Omega,\mathcal F)$.
Then
\[
\mathrm{KL}(\psi_\sharp P \,\|\, \psi_\sharp Q) \le \mathrm{KL}(P\|Q).
\]
\end{lemma}

\begin{lemma}[Two-variable chain rule for KL {\cite[Lemma 40]{suggala2024second}}]] \label{lem:two-var-chain-rule}
Let $(X,Y)$ have joint laws $P_{XY}$ and $Q_{XY}$ on a common measurable space with $P_{XY}\ll Q_{XY}$.
Let $P_X,Q_X$ be the marginals, and $P_{Y\mid X},Q_{Y\mid X}$ be regular conditional distributions.
Then
\[
\mathrm{KL}(P_{XY}\|Q_{XY})
=
\mathrm{KL}(P_X\|Q_X)
+
\mathbb{E}_{X\sim P_X}\Big[\mathrm{KL}\big(P_{Y\mid X}(\cdot\mid X)\,\|\,Q_{Y\mid X}(\cdot\mid X)\big)\Big].
\]
\end{lemma}

\subsection{Proof of Theorem \ref{thm:sec25_deploy}}
\label{apx:first_proof}

\subsubsection{Assumptions}

\begin{assumption}[Bounded increments]\label{ass:bounded}
Let \(Q\coloneqq \P^\star_{\pi_\star}\). There exist constants $a\in\R$ and $B'>0$ such that, $Q$-almost surely, for all
$t\in\{1,\dots,T\}$, the per-round regret
\(
r_t \in [a,a+B'].
\)
Equivalently, the range length satisfies $\sup r_t - \inf r_t \le B'$ a.s.
\end{assumption}

\begin{assumption}[Deterministic predictable drift under $Q$]\label{ass:det_drift}
Let \(Q\coloneqq \P^\star_{\pi_\star}\).
For each $t\in\{1,\dots,T\}$, the conditional mean
\(
\mu_t^Q := \E_Q[r_t \mid \mathcal H_{t-1}]
\)
is $Q$-almost surely deterministic. In other words, there exists a constant $m_t\in\R$ such that
$\mu_t^Q = m_t$ $Q$-a.s. Equivalently, we can also say $\E_Q[r_t \mid \mathcal H_{t-1}] = \E_Q[r_t]$ $Q$-a.s.
\end{assumption}

\begin{assumption}[Bounded likelihood ratio]
\label{assump:B-comp-overlap}
There exists \(V_{\mathrm{comp}}\in[1,\infty)\) such that for all \(t\) and all histories \(h\in\mathcal{H}_{t-1}\),
\(Q_{t,\theta}(\cdot\mid h)\ll Q^\star_t(\cdot\mid h)\) and
\[
\sup_{F\in\mathcal{F}_{\mathcal{U}}}
\frac{Q_{t,\theta}(F\mid h)}{Q^\star_t(F\mid h)}
\;\le\;
V_{\mathrm{comp}}.
\]
\end{assumption}

\paragraph{Discussion.}
Assumption~\ref{ass:det_drift} states that, under the benchmark measure $\P^\star$, the one-step
\emph{predictable drift} of the adapted sequence is \emph{history-independent}:
for each $t$, the conditional mean $\E_Q[r_t\mid\mathcal H_{t-1}]$ is (almost surely) a
deterministic constant $m_t$. Intuitively, this rules out shocks in the filtration that could simultaneously shift many
future increments through their conditional expectations. Practically, such an assumption is natural whenever $\P^\star$ describes a benchmark
data-generating process that is exogenous to the learner’s history, e.g., the environment noise is i.i.d.\ under $\P^\star$, or more generally, if the benchmark makes the conditional expectation time-dependent but history-independent.

Assumption~\ref{assump:B-comp-overlap} is a uniform support or overlap requirement between the learned completion kernel and the true completion kernel. Intuitively, this rules out missing mass and prevents extreme underestimation of events that can occur under $Q_t^\star$. In real scenarios, this is plausible when the model class is constrained to have the same effective support as the true completion process. The main implication for our proof is that this boundedness makes a change-of-measure possible by converting the \emph{reverse} KL terms that arise naturally under deployment into \emph{forward} KL terms that we can control via the on-policy completion mismatch, incurring only a constant logarithmic factor.

\subsubsection{Supporting Lemmas}
\label{app:thm2.4:supporting_lemmsa}

\begin{lemma}[MGF bound for $R_T$]\label{lem:mgf-adapted}
Under Assumptions~\ref{ass:bounded} and~\ref{ass:det_drift}, for all $\lambda\in\R$, we have
\[
\log \E_Q\!\left[e^{\lambda R_T}\right]
\le
\lambda\,\E_Q[R_T] \;+\; \frac{\lambda^2 T(B')^2}{8}.
\]
\end{lemma}

\begin{proof}
By Assumption~\ref{ass:det_drift}, define the deterministic constants
$m_t := \E_Q[r_t\mid\mathcal H_{t-1}]$. Also define
\(
d_t := r_t - m_t.
\)
Then $\E_Q[d_t\mid\mathcal H_{t-1}]=0$ $Q$-a.s.
By Assumption~\ref{ass:bounded}, $r_t\in[a,a+B']$ a.s., hence $d_t$ lies in an interval
of length $B'$ a.s. Therefore Lemma~\ref{lem:md_mgf}, with $c = B'$, gives
\[
\log \E_Q\!\left[\exp\!\Big(\lambda\sum_{t=1}^T d_t\Big)\right]
\le
\frac{\lambda^2 T(B')^2}{8}.
\]
Since $\sum_{t=1}^T d_t = R_T - \sum_{t=1}^T m_t$, we have
\(
\E_Q[e^{\lambda R_T}]
= e^{\lambda\sum_{t=1}^T m_t}\;
\E_Q\!\left[\exp\!\Big(\lambda\sum_{t=1}^T d_t\Big)\right],
\)
and taking logs yields
\[
\log \E_Q[e^{\lambda R_T}]
=
\lambda\sum_{t=1}^T m_t
+
\log \E_Q\!\left[\exp\!\Big(\lambda\sum_{t=1}^T d_t\Big)\right]
\le
\lambda\sum_{t=1}^T m_t + \frac{\lambda^2 T(B')^2}{8}.
\]
Finally, because each $m_t$ is deterministic,
\(
\sum_{t=1}^T m_t
= \sum_{t=1}^T \E_Q[r_t\mid\mathcal H_{t-1}]
= \sum_{t=1}^T \E_Q[r_t]
= \E_Q[R_T],
\)
where we used the tower property $\E_Q[\E_Q[r_t\mid\mathcal H_{t-1}]]=\E_Q[r_t]$ and the
fact that $\E_Q[r_t\mid\mathcal H_{t-1}]=m_t$ is constant.
Substituting completes the proof.
\end{proof}

\begin{lemma}[Trajectory KL equals causal action mismatch]
\label{lem:B7-causal-chain-rule}
Let \(\pi_\theta\) and \(\pi_\star\) be two policies interacting with the \emph{same} environment kernels
\(\{M^\star_t\}_{t=1}^T\).
Then we have the following equality below
\[
\mathrm{KL}\!\left(\P^\star_{\pi_\theta}\,\middle\|\,\P^\star_{\pi_\star}\right)
=
\Delta_{T}^{\mathrm{act}}(\theta).
\]
\end{lemma}

\begin{proof}
Under \(\P^\star_{\pi_\theta}\) and \(\P^\star_{\pi_\star}\), the joint density or kernel factorization over \(Z_{1:T}\)
is, informally,
\(
\P^\star_{\pi}(dz_{1:T})
=
\prod_{t=1}^T
K^{\pi}_t(dx_t\mid h_{t-1})\,
M^\star_t(do_t\mid h_{t-1},x_t),
\)
where $z_t=(x_t,o_t)$ and \(h_{t-1}=z_{1:t-1}\).
Since \(M^\star_t\) is identical under both policies, the Radon--Nikodym derivative
\(d\P^\star_{\pi_\theta}/d\P^\star_{\pi_\star}\) is the product of the action-kernel ratios only, hence by the
chain rule for KL divergence (applied iteratively with conditioning on \(H_{t-1}\)) we obtain
\[
\mathrm{KL}\!\left(\P^\star_{\pi_\theta}\,\middle\|\,\P^\star_{\pi_\star}\right)
=
\sum_{t=1}^T
\E_{\P^\star_{\pi_\theta}}\!\left[
\mathrm{KL}\!\Big(
K^{\pi_\theta}_t(\cdot\mid H_{t-1})
\,\Big\|\,
K^{\pi_\star}_t(\cdot\mid H_{t-1})
\Big)\right].
\]
Finally we can substitute \(K^{\pi_\theta}_t=K_{t,\theta}\) and \(K^{\pi_\star}_t=K^\star_t\) to obtain the desired result of the Lemma.
\end{proof}

\begin{lemma}[Causal action mismatch bounded by reverse completion mismatch]
\label{lemma:B12-act-by-rev-comp}
For every \(\theta\),
\[
\Delta_{T}^{\mathrm{act}}(\theta)
\;\le\;
\Delta^{\mathrm{comp}}_{T, \, \mathrm{rev}}(\theta;\pi_\theta).
\]
\end{lemma}

\begin{proof}
Fix \(t\) and condition on \(H_{t-1}=h\).
By construction, \(K^\star_t(\cdot\mid h)\) and \(K_{t,\theta}(\cdot\mid h)\) are pushforwards of
\(Q^\star_t(\cdot\mid h)\) and \(Q_{t,\theta}(\cdot\mid h)\) through the measurable map
\(\varphi_t(\cdot,h)\).
Thus the data-processing inequality for KL divergence (Lemma \ref{lem:dp-kl}) yields
\[
\mathrm{KL}\!\Big(
K_{t,\theta}(\cdot\mid h)\,\Big\|\,K^\star_t(\cdot\mid h)
\Big)
\;\le\;
\mathrm{KL}\!\Big(
Q_{t,\theta}(\cdot\mid h)\,\Big\|\,Q^\star_t(\cdot\mid h)
\Big).
\]
Taking expectation \(\E_{H_{t-1}\sim \P^\star_{\pi_\theta}}\) and summing over \(t=1,\dots,T\) gives the desired claim of the Lemma.
\end{proof}

\begin{lemma}[Reverse completion mismatch is controlled by forward completion mismatch]
\label{cor:B13-rev-by-fwd-comp}
Under Assumption~\ref{assump:B-comp-overlap},
\[
\Delta^{\mathrm{comp}}_{T, \, \mathrm{rev}}(\theta;\pi_\theta)
\;\le\;
(2+\log V_{\mathrm{comp}})\,
\Delta^{\mathrm{comp}}_{T}(\theta;\pi_\theta).
\]
\end{lemma}

\begin{proof}
Fix \(t\) and condition on \(H_{t-1}=h\).
Let \(P=Q_{t,\theta}(\cdot\mid h)\) and \(Q=Q^\star_t(\cdot\mid h)\).
By Assumption~\ref{assump:B-comp-overlap} we have \(\sup_{F}P(F)/Q(F)\le V_{\mathrm{comp}}\),
so Lemma~\ref{lem:kl-le-hellinger} implies that the reverse KL is bounded as
\begin{align*}
\mathrm{KL}(P\|Q) &\le
(2+\log V_{\mathrm{comp}})\,D^2_{\mathrm{H}}(P,Q) \\
&= (2+\log V_{\mathrm{comp}})\,D^2_{\mathrm{H}}(Q,P) \\
&\leq (2+\log V_{\mathrm{comp}})\mathrm{KL}(Q\|P).
\end{align*}
Here, the first equality is due to the fact that the squared Hellinger distance is symmetric: \(D^2_{\mathrm{H}}(P,Q)=D^2_{\mathrm{H}}(Q,P)\). The second inequality is due to Lemma \ref{lem:hellinger-le-kl}, \(D^2_{\mathrm{H}}(Q,P)\le \mathrm{KL}(Q\|P)\).
Combining, 
\[
\mathrm{KL}\!\Big(Q_{t,\theta}(\cdot\mid h)\,\big\|\,Q^\star_t(\cdot\mid h)\Big)
\;\le\;
(2+\log V_{\mathrm{comp}})\,
\mathrm{KL}\!\Big(Q^\star_t(\cdot\mid h)\,\big\|\,Q_{t,\theta}(\cdot\mid h)\Big).
\]
Taking expectation \(\E_{H_{t-1}\sim \P^\star_{\pi_\theta}}\) and summing over \(t\) gives the stated inequality in the Lemma.
\end{proof}

\subsubsection{Proof of the theorem}
\label{app:thm2.4:main_proof}

\begin{theorem*}[Deployment penalty bound under $\P^\star$]
\label{thm:sec25_deploy_apx}
Let $\pi_\star=\pi_{\P^\star}$ be the ideal GPS policy that uses the true completion kernel, and
let $\pi_\theta=\pi_{p_\theta}$ be the deployed GPS policy using the learned kernel. Assume Assumption \ref{assump:B-comp-overlap} and define $B^\prime := B \max \{h, b\}$. Then
\begin{equation*}
\BReg_T(\pi_\theta;\P^\star)
\le
\BReg_T(\pi_\star;\P^\star)
+
B^\prime \sqrt{\frac{T(2+\log V_{\mathrm{comp}})}{2}\,\Delta^{\mathrm{comp}}_{T}(\theta;\pi_\theta)}.
\end{equation*}
\end{theorem*}

\begin{proof}[Proof of Theorem \ref{thm:sec25_deploy}]
Let \(P\coloneqq \P^\star_{\pi_\theta}\) and \(Q\coloneqq \P^\star_{\pi_\star}\). Let us denote the per-round regret as $r_t := f^\star(X_t) - \min_{x \in \mathcal X}f^\star(x)$. The cumulative regret is denoted as $R_T := \sum_{t=1}^T r_t$, and the cumulative Bayesian regret for a policy $\pi$ under environment $\P^\star$ is denoted by $\mathrm{BayesReg}_T(\pi;\P^\star) := \mathbb E_{\P^\star_{\pi}}\!\big[R_T\big]$. We will proceed by first deriving a bound on \(\E_{P}[R_T]-\E_{Q}[R_T].\)

The first step involves change of measure. This is done using the Donsker–Varadhan (DV) variational formula from Lemma~\ref{lem:DV}. Using the Lemma with $f=\lambda R_T$, for any $\lambda > 0$ gives $\lambda\,\mathbb E_P[R_T] \le \mathrm{KL}(P\|Q) + \log \mathbb E_Q[e^{\lambda R_T}]$. Hence, dividing throughout by $\lambda$ gives us
\begin{equation}
\label{eq:dv-step}
\mathbb E_P[R_T]
\le
\frac{1}{\lambda}\mathrm{KL}(P\|Q) + \frac{1}{\lambda}\log \mathbb E_Q[e^{\lambda R_T}].
\end{equation}

The second step involves bounding the MGF under the benchmark $Q$. Under $Q$, the process is generated by $\pi_\star$ in environment $\P^\star$, and $R_T=\sum_{t=1}^T r_t(Z_t)$ with each $r_t\in[a, a + B^\prime]$ adapted to the natural filtration.
Lemma~\ref{lem:mgf-adapted} yields
\[
\log \mathbb E_Q[e^{\lambda R_T}]
\le
\lambda\,\mathbb E_Q[R_T] + \frac{\lambda^2 T (B^\prime)^2}{8}.
\]
Plugging into \eqref{eq:dv-step} gives, for all $\lambda > 0$
\begin{equation}
\label{eq:DV-hoeffding_1}
\mathbb E_P[R_T] - \mathbb E_Q[R_T] \le \frac{1}{\lambda}\mathrm{KL}(P\|Q) + \frac{\lambda T (B^\prime)^2}{8}.
\end{equation}

The third step involves choosing the optimal value of $\lambda$. Minimizing the right hand side of \eqref{eq:DV-hoeffding_1} over \(\lambda>0\) gives
\(
\lambda^\star=\sqrt{\frac{8\mathrm{KL}(P\|Q)}{T (B^\prime)^2}}
\). Plugging in the optimal value of $\lambda$ into \eqref{eq:DV-hoeffding_1} gives the following bound
\begin{equation}
\label{eq:sqrtT-KL_1}
\E_{P}[R_T]-\E_{Q}[R_T]
\;\le\;
B^\prime \sqrt{\frac{T}{2}\,\mathrm{KL}(P\|Q)}.
\end{equation}

In the fourth step, recall that from our notations $\mathbb E_P[R_T]=\mathrm{BayesReg}_T(\pi_\theta;\P^\star)$ and $\mathbb E_Q[R_T]=\mathrm{BayesReg}_T(\pi_\star;\P^\star)$. Moreover, Lemma~\ref{lem:B7-causal-chain-rule}, gives us
$\mathrm{KL}(\P_{\pi_\theta}^\star\|\P_{\pi_\star}^\star)=\Delta_{T}^{\mathrm{act}}(\theta)$. Applying these to \eqref{eq:sqrtT-KL_1} gives

\begin{equation}
\label{eq:step4}
\mathrm{BayesReg}_T(\pi_\theta;\P^\star) \leq \mathrm{BayesReg}_T(\pi_\star;\P^\star) + B^\prime \sqrt{\frac{T}{2} \Delta_{T}^{\mathrm{act}}(\theta)}
\end{equation}

Finally, we note that the casual action mismatch by the reverse completion mismatch, which in turn is conreolled by the forward completion mismatch, under Assumption \ref{assump:B-comp-overlap} (Lemma \ref{lemma:B12-act-by-rev-comp} and \ref{cor:B13-rev-by-fwd-comp}). Applying these to \eqref{eq:step4}, gives us the desired statement for this theorem as follows

\begin{equation*}
\mathrm{BayesReg}_T(\pi_\theta;\P^\star) \leq \mathrm{BayesReg}_T(\pi_\star;\P^\star) + B^\prime \sqrt{\frac{T\left(2+\log V_{\mathrm{comp}}\right)}{2}\,\Delta^{\mathrm{comp}}_{T}(\theta;\pi_\theta)}\,.
\end{equation*}

\end{proof}

\subsection{Proof of Theorem \ref{thm:rnv_bayes_regret}}
\label{app:proof_thm_ts_newsvendor}

Fix cost parameters \(h>0\) (overage) and \(b>0\) (underage). In each round \(t=1,2,\dots,T\), the decision-maker chooses an order quantity \(X_t \in \cX := [0,B]\), then demand \(D_t \in [0,B]\) is realized, and the incurred (unobserved) newsvendor loss is \(\ell(x,d) \;:=\; h(x-d)^+ + b(d-x)^+.\) The observable (censored) feedback is
\(
S_t := \min\{D_t,X_t\}, C_t := \mathbf{1}\{D_t \le X_t\}, O_t := (S_t,C_t).
\)
Note that if \(C_t=1\) then \(D_t=S_t\) is fully observed; if \(C_t=0\) (stockout) then \(D_t>X_t\) is censored.

Let the prior be denoted by $\xi$, which could be random in the case of a meta-bandit setting \citep{kveton2021meta}. The demand law is $\P \sim \xi$, and the per round demand is $D_t \overset{\mathrm{i.i.d.}}{\sim} \P
$ on \([0,B]\). The conditional risk is defined as $f_\P(x) := \mathbb{E}_\P[\ell(x,D)]$, with minimizer \(x_\P^\star \in \arg\min_{x\in[0,B]} f_\P(x)\). For a policy \(\pi\), define Bayesian regret \(\mathrm{BayesReg}_T(\pi) := \mathbb{E}\Big[\sum_{t=1}^T \big(f_\P(X_t)-f_\P(x_\P^\star)\big)\Big]\).
At each round \(t\), Thompson sampling (TS) samples \(\tilde\P_t\) from the posterior of \(\xi\) given history \(H_{t-1}\), according to $\tilde\P_t \sim \xi(\cdot \mid H_{t-1})$, and plays the action \(X_t \in \arg\min_{x\in[0,B]} f_{\tilde\P_t}(x)\). We now prove Theorem~\ref{thm:rnv_bayes_regret} via a reduction to one-dimensional bandit convex optimization (BCO) using an \emph{observable} unbiased surrogate loss constructed from censored feedback.

\subsubsection{Supporting Lemmas}

\begin{lemma}[Risk-equivalent objective $g_\P$]\label{lem:shift_objective}
For any \(\P\), define
\(g_\P(x) := (h+b)\,\mathbb{E}_\P\big[(x-D)^+] - b x\), where \(x\in[0,B].\) Then, \(f_\P(x) \;=\; g_\P(x) + b\,\mathbb{E}_\P[D]\), for all $x\in[0,B]$.
Consequently, for any policy \(\pi\), the Bayesian regret can be analyzed using $g_\P$ instead of $f_\P$, since
\[
\mathbb{E}\Big[\sum_{t=1}^T \big(f_\P(X_t)-f_\P(x_\P^\star)\big)\Big]
\;=\;
\mathbb{E}\Big[\sum_{t=1}^T \big(g_\P(X_t)-g_\P(x_\P^\star)\big)\Big].
\]
\end{lemma}

\begin{proof}
Using the identity \((d-x)^+ = (d-x) + (x-d)^+\), valid for all \(x,d\in\mathbb{R}\), we have
\begin{align*}
\ell(x,d)
&= h(x-d)^+ + b\big((d-x)+(x-d)^+\big) \\
&= (h+b)(x-d)^+ + b(d-x). 
\end{align*}

Taking conditional expectation given \(\xi=\P\) yields the following
\begin{align*}
f_\P(x)
&= (h+b)\,\mathbb{E}_\P\big[(x-D)^+\big] + b\big(\mathbb{E}_\P[D]-x\big)\\
&= g_\P(x) + b\,\mathbb{E}_\P[D].
\end{align*}

The above proves the first claim of the Lemma. Recognizing that the additive term \(b\,\mathbb{E}_\P[D]\) does not depend on \(x\), hence cancels in regret and does not change minimizers, gives the final equality of the Bayesian regret terms.
\end{proof}

\begin{lemma}[Randomization identity]\label{lem:randomization_identity}
Fix \(x\in[0,B]\). Let \(U \sim \mathrm{Unif}([0,x])\) be a random variable independent of \(D\), with the convention that if \(x=0\) then \(U\equiv 0\). Also define \(Z := \mathbf{1}\{D \le U\}\). Then
\[
\mathbb{E}[x Z] \;=\; \mathbb{E}[(x-D)^+].
\]
\end{lemma}

\begin{proof}
The claim is immediate for \(x=0\). So, let us fix \(x>0\). Conditioning on \(D\), for any realization \(D=d\) we have
\(
\mathbb{E}[Z\mid D=d] = \mathbb{P}(U\ge d\mid D=d)
= \frac{(x-d)^+}{x},
\)
since \(U\) is uniform on \([0,x]\).
Therefore, we have
\(\mathbb{E}[xZ\mid D] = x\,\mathbb{E}[Z\mid D] = (x-D)^+,\) where the first equality is due to the fact that $x$ is deterministic here, and can therefore be taken outside the expectation. Finally, taking expectation again and applying the tower rule of conditional expectations gives \(\mathbb{E}[xZ]= \mathbb{E}[\mathbb{E}[xZ\mid D]] = \mathbb{E}[(x-D)^+]\).
\end{proof}

\begin{lemma}[Observable unbiased feedback for \(g_\P\)]\label{lem:unbiased_feedback}
Fix \(x\in[0,B]\) and define \(U,Z\) as in Lemma~\ref{lem:randomization_identity}. Let us define a one-step bandit feedback as \(Y := (h+b)\,x Z - b x.\) Then, for every \(\P\),
\[
\mathbb{E}_\P[Y] \;=\; g_\P(x).
\]
\end{lemma}

\begin{proof}
From the definition of $Y$, we have
\begin{align*}
\mathbb{E}_\P[Y]
&= (h+b)\,\mathbb{E}_\P[xZ] - b x\\
&= (h+b)\,\mathbb{E}_\P[(x-D)^+] - b x\\
&= g_\P(x),
\end{align*}
where in the second step we used Lemma~\ref{lem:randomization_identity}.
\end{proof}

\begin{lemma}[Computability of \(Z\) from censored observation]\label{lem:Z_computable}
Fix \(x\in[0,B]\). Let \(O=(S,C)\) be the censored observation from ordering \(x\):
\(S=\min\{D,x\}\), \(C=\mathbf{1}\{D\le x\}\).
Let \(U\sim\mathrm{Unif}([0,x])\) be sampled by the decision-maker after observing \(O\),
and define \(Z=\mathbf{1}\{D\le U\}\).
Then \(Z\) is measurable with respect to \((O,U)\), i.e., \(Z\) can be computed from \((S,C,U)\).
\end{lemma}

\begin{proof}
If \(C=0\), then \(D>x\). Moreover, since  \(U\in[0,x]\), we have $x \geq U$, implying \(Z=0\) for this case. On the other hand, if \(C=1\), then \(D\le x\) and in fact \(D=S\), hence \(Z=\mathbf{1}\{S\le U\}\). Thus \(Z\) is computable from \((S,C,U)\) in all cases.
\end{proof}

\begin{lemma}[Convexity, Lipschitzness, and bounded feedback]\label{lem:g_properties}
For every \(\P\), the function \(g_\P:[0,B]\to\mathbb{R}\) is convex and
\(L\)-Lipschitz with \(L:=\max\{h,b\}\).
Moreover, for any \(x\in[0,B]\), the random feedback \(Y=(h+b)xZ-bx\) satisfies the uniform bounds $Y \in [-bB, hB]$.
\end{lemma}

\begin{proof}
Define \(\Phi_\P(x):=\mathbb{E}_\P[(x-D)^+]\).
For each fixed \(d\), the map \(x\mapsto (x-d)^+\) is convex, and expectation preserves convexity,
so \(\Phi_\theta\) is convex; hence \(g_\P(x)=(h+b)\Phi_\theta(x)-bx\) is also convex.

To obtain Lipschitzness, note that for each \(d\), any subgradient of \(x\mapsto(x-d)^+\) lies in \([0,1]\),
so any subgradient of \(\Phi_\theta\) lies in \([0,1]\) as well.
Therefore any subgradient \(s(x)\in\partial g_\P(x)\) has the form \(s(x)=(h+b)\,\xi(x)-b,\) for some $\xi(x)\in[0,1]$ and hence \(s(x)\in[-b,h]\). This implies \(|s(x)|\le \max\{h,b\}=L\). Finally, using the subgradient inequality on $g_\P(y)$ and $g_\P(x)$, we conclude that \(g_\P\) is \(L\)-Lipschitz.

To obtain the boundedness of the feedback $Y$, recall that since \(Z\in\{0,1\}\) and \(x\le B\), we have
\[
Y = x\big((h+b)Z-b\big)
=
\begin{cases}
-bx \in[-bB,0], & Z=0,\\
hx \in[0,hB], & Z=1,
\end{cases}
\]
which gives the desired bound of $Y$ as \(-bB\le Y\le hB\). This completes the proof of the Lemma.
\end{proof}

\subsubsection{Proof of the Theorem}

\begin{theorem*}[Bayesian regret of TS for the repeated newsvendor]\label{thm:ts_newsvendor}
Assume demands satisfy \(D_t\in[0,B]\) almost surely. Let \(\pi^\star\) be the Thompson sampling policy that
samples \(\tilde\P_t\) from the posterior of \(\xi\) given \(H_{t-1}\) and plays
\(X_t\in\arg\min_{x\in[0,B]} f_{\tilde\P_t}(x)\).
Then
\[
\mathrm{BayesReg}_T(\pi^\star) \;\le\; \tilde O(\sqrt{T}),
\]
where \(\tilde O(\cdot)\) suppresses polylogarithmic factors and constants depending on \(h,b,B\).
\end{theorem*}

\begin{proof}[Proof of Theorem \ref{thm:rnv_bayes_regret}]
By Lemma~\ref{lem:shift_objective}, it suffices to bound
\(\mathbb{E}\big[\sum_{t=1}^T (g_\P(X_t)-g_\P(x_\P^\star))\big]\) in order to bound $\mathrm{BayesReg}_T(\pi^\star)$. The proof works by first showing that the R-NV admits the same feedback structure as bandit convex optimization (BCO) in \citet{bakhtiari2025bco}, and then using their Bayesian regret bound for Thompson Sampling in the case of one-dimensional actions \cite[Theorem 4]{bakhtiari2025bco}.

The first step involves reducing the repeated newsvendor problem to a one-dimensional bandit convex optimization instance. Fix any round \(t\). After choosing \(X_t\in[0,B]\) and observing censored feedback \(O_t=(S_t,C_t)\),
sample \(U_t\sim\mathrm{Unif}([0,X_t])\) (degenerate at 0 if \(X_t=0\)), and compute \(Z_t\) from \((O_t,U_t)\)
as in Lemma~\ref{lem:Z_computable}. Define the scalar feedback
\(
Y_t := (h+b)\,X_t Z_t - b X_t.
\)
Then Lemma~\ref{lem:unbiased_feedback} ensures that conditional on \(\P\) and \(X_t\),
\(
\mathbb{E}[Y_t \mid \P, X_t] = g_\P(X_t),
\)
and Lemma~\ref{lem:g_properties} gives that \(g_\P\) is convex, \(L\)-Lipschitz, and \(Y_t\in[-bB,hB]\).

The second step involves, normalization to the canonical BCO class. Define the rescaled action \(\bar X_t := X_t/B \in [0,1]\) and the normalized loss feedback
\(
\bar Y_t := \frac{Y_t + bB}{(h+b)B} \in [0,1].
\)
For each \(\P\), define the normalized convex loss
\(
\bar g_\P(u) := \frac{g_\P(Bu)+bB}{(h+b)B}, u\in[0,1].
\)
Then \(\tilde g_\P\) is convex, takes values in \([0,1]\), and is \(1\)-Lipschitz on \([0,1]\)
(because any subgradient of \(g_\P\) lies in \([-b,h]\), hence any subgradient of \(\tilde g_\P\)
lies in $\big[-\frac{b}{h+b},\frac{h}{h+b}\big]\subseteq[-1,1])$.
Moreover,
\(
\mathbb{E}[\bar Y_t \mid \P=\theta, \bar X_t=u]
= \bar g_\P(u).
\)

The third step is to recognize the equivalence between TS for outcome and function sampling, that has been detailed in Appendix \ref{subsec:continuous}. The prior on \(\P\) induces a prior on the (random) convex function \(\tilde g_\P\).
The TS policy \(\pi^\star\) samples \(\tilde\P_t\) from the posterior and plays a minimizer of
\(f_{\tilde\theta_t}\), equivalently a minimizer of \(g_{\tilde\theta_t}\) by Lemma~\ref{lem:shift_objective},
and equivalently a minimizer of \(\tilde g_{\tilde\theta_t}\) after the above scaling.
Therefore, \(\pi^\star\) is precisely Thompson sampling for the normalized one-dimensional BCO instance with convex 1-Lipschitz losses in \([0,1]\) and bandit feedback \(\bar Y_t\).

Finally, according to \cite[Theorem 4]{bakhtiari2025bco}, Thompson sampling in one-dimensional bandit convex optimisation has Bayesian regret \(\tilde O(\sqrt{T})\) under convexity, boundedness, and a Lipschitz assumption. Applying this to the normalized instance yields
\(
\mathbb{E}\Big[\sum_{t=1}^T \big(\tilde g_\P(\bar X_t) - \tilde g_\P(\tilde x_\P^\star)\big)\Big]
\;\le\; \tilde O(\sqrt{T}),
\), where \(\tilde x_\P^\star\in\arg\min_{u\in[0,1]}\bar g_\P(u)\) corresponds to \(x_\P^\star\) under scaling.
Undoing the normalization from above gives
\[
\mathbb{E}\Big[\sum_{t=1}^T \big(g_\P(X_t)-g_\P(x_\P^\star)\big)\Big]
\;\le\; (h+b)B \cdot \tilde O(\sqrt{T})
\;=\; \tilde O(\sqrt{T}),
\]
absorbing \((h+b)B\) into the \(\tilde O(\cdot)\) notation.
Lastly, Lemma~\ref{lem:shift_objective} transfers this bound back to \(f_\P\), completing the proof.
\end{proof}

\subsection{Proof of Theorem \ref{thm:completion_to_ovserved}}
\label{app:proof-thm-4.8}

Fix a deployed policy $\pi=\pi_\theta$. For each history value $h$ and action $x\in[0,B]$, let
$P^{\mathrm{obs}}_{h,x}$ denote the true conditional law of the censored observation
$O_t$ given $(H_{t-1}=h,X_t=x)$, and let $Q^{\mathrm{obs}}_{h,x}$ denote the model-implied conditional law
under parameter $\theta$. Assume that for $\P^\star_\pi$-a.e.\ realized $(H_{t-1},X_t)$ we have
$P^{\mathrm{obs}}_{H_{t-1},X_t}\ll Q^{\mathrm{obs}}_{H_{t-1},X_t}$.

Define the one-step observed log-losses
\(
\ell^{\mathrm{obs}}_\theta(o\mid h,x)
:= -\log p_{\theta,\mathrm{obs}}(o\mid h,x),
\ell^{\mathrm{obs}}_\star(o\mid h,x)
:= -\log p^\star_{\mathrm{obs}}(o\mid h,x),
\)
where $p_{\theta,\mathrm{obs}}(\cdot\mid h,x)$ and $p^\star_{\mathrm{obs}}(\cdot\mid h,x)$ are densities
of $Q^{\mathrm{obs}}_{h,x}$ and $P^{\mathrm{obs}}_{h,x}$ w.r.t.\ a common dominating measure on $O=[0,B]\times\{0,1\}$. Define the cumulative observed predictive objectives
\(
\mathcal{L}^{\mathrm{obs}}_T(\theta;\pi)
:= \mathbb{E}_{\P^\star_\pi}\!\Big[\sum_{t=1}^T \ell^{\mathrm{obs}}_\theta(O_t\mid H_{t-1},X_t)\Big],\) and \(\mathcal{L}^{\mathrm{obs}}_T(\star;\pi)
:= \mathbb{E}_{\P^\star_\pi}\!\Big[\sum_{t=1}^T \ell^{\mathrm{obs}}_\star(O_t\mid H_{t-1},X_t)\Big],
\)
and recall the observed mismatch
\(
\Delta^{\mathrm{obs}}_T(\theta;\pi)
:= \sum_{t=1}^T \mathbb{E}_{(H_{t-1},X_t)\sim \P^\star_\pi}
\Big[\mathrm{KL}\!\big(P^{\mathrm{obs}}_{H_{t-1},X_t}\,\|\,Q^{\mathrm{obs}}_{H_{t-1},X_t}\big)\Big].
\)

\subsubsection{Assumptions}
\label{apx:thm_3_asm}

\begin{assumption}[Max-order coverage]\label{ass:coverage}
There exists $\eta\in(0,1]$ such that along the deployed interaction (under $P^\star$ and $\pi_\theta$),
\[
\mathbb{P}^\star_{\pi_\theta}(X_t=B \mid H_{t-1}) \ge \eta\qquad\text{a.s. for all }t.
\]
\end{assumption}

\paragraph{Discussion.}
Assumption~\ref{ass:coverage} postulates a uniform coverage condition: under deployment, for every history $h\in H_{t-1}$,
the policy selects the maximal order $X_t=B$ with probability at least $\eta>0$. In our censored setting,
the event $\{X_t=B\}$ yields fully informative observations because the censoring map does not
truncate demand at the boundary, so a nonvanishing fraction of rounds directly captures the latent demand
mechanism. This mirrors the observable boundary phenomenon in \citet{hssaine2024censored}.

\begin{assumption}[Self-contraction along AR pseudo-generation]\label{ass:self-contraction}
There exists $c_{\mathrm{sc}}\in(0,\infty)$ such that for every history value $h$ and every $s\ge 2$,
\[
\delta_s(h) \le \frac{c_{\mathrm{sc}}}{s}\, d(h),
\]
where $d(h)$ is the one-step divergence defined in \eqref{eq:def-dh} below, and $\delta_s(h)$ is the $s$-th AR conditional divergence
defined in \eqref{eq:def-delta-sh} below.
\end{assumption}

\paragraph{Discussion}
Assumption~\ref{ass:self-contraction} is a non-accumulation condition for autoregressive (AR) pseudo-generation in the completion model. The assumption requires a harmonic decay, meaning that errors in later AR conditionals are controlled by (and shrink relative to) the first-step mismatch, rather than compounding linearly with rollout length. Intuitively, this captures a self-correcting regime: as pseudo-history grows, conditioning becomes more informative, and the AR process becomes increasingly stable to small perturbations. The key implication for our proof is that this assumption yields a sharp control of the full completion KL by the one-step mismatch up to a mild $\log T$ factor.

\subsubsection{Supporting Lemmas}
\label{apx:lemmalist_thm3}
\begin{lemma}[Observed KL mismatch equals excess observed log-loss]\label{lem:delta-obs-excess-Lobs}
We have the identity
\[
\Delta^{\mathrm{obs}}_T(\theta;\pi)
=
\mathcal{L}^{\mathrm{obs}}_T(\theta;\pi)
-
\mathcal{L}^{\mathrm{obs}}_T(\star;\pi)
\;\;\ge\; 0.
\]
\end{lemma}

\begin{proof}
Fix $t\in\{1,\dots,T\}$ and condition on $(H_{t-1},X_t)=(h,x)$. By the definition of conditional KL divergence (with the absolute continuity assumption above), we have the following identity
\begin{align*}
\mathrm{KL}\!\big(P^{\mathrm{obs}}_{h,x}\,\|\,Q^{\mathrm{obs}}_{h,x}\big)
&=
\mathbb{E}\!\left[
\log\!\left(
\frac{p^\star_{\mathrm{obs}}(O_t\mid h,x)}{p_{\theta,\mathrm{obs}}(O_t\mid h,x)}
\right)
\,\Big|\, H_{t-1}=h, X_t=x
\right] \\
&=
\mathbb{E}\!\left[
-\log p_{\theta,\mathrm{obs}}(O_t\mid h,x) + \log p^\star_{\mathrm{obs}}(O_t\mid h,x)
\,\Big|\, H_{t-1}=h, X_t=x
\right] \\
&=
\mathbb{E}\!\left[
\ell^{\mathrm{obs}}_\theta(O_t\mid h,x) - \ell^{\mathrm{obs}}_\star(O_t\mid h,x)
\,\Big|\, H_{t-1}=h, X_t=x
\right].
\end{align*}
Now take expectation over $(H_{t-1},X_t)\sim \P^\star_\pi$ and apply the tower property to obtain the following
\[
\mathbb{E}_{(H_{t-1},X_t)\sim \P^\star_\pi}
\Big[\mathrm{KL}\!\big(P^{\mathrm{obs}}_{H_{t-1},X_t}\,\|\,Q^{\mathrm{obs}}_{H_{t-1},X_t}\big)\Big]
=
\mathbb{E}_{\P^\star_\pi}\!\big[\ell^{\mathrm{obs}}_\theta(O_t\mid H_{t-1},X_t) - \ell^{\mathrm{obs}}_\star(O_t\mid H_{t-1},X_t)\big].
\]
Finally, summing over $t=1,\dots,T$ and using linearity of expectation yields the desired identity
\[
\Delta^{\mathrm{obs}}_T(\theta;\pi)
=
\mathbb{E}_{\P^\star_\pi}\!\Big[\sum_{t=1}^T \ell^{\mathrm{obs}}_\theta(O_t\mid H_{t-1},X_t)\Big]
-
\mathbb{E}_{\P^\star_\pi}\!\Big[\sum_{t=1}^T \ell^{\mathrm{obs}}_\star(O_t\mid H_{t-1},X_t)\Big]
=
\mathcal{L}^{\mathrm{obs}}_T(\theta;\pi) - \mathcal{L}^{\mathrm{obs}}_T(\star;\pi).
\]
Moreover, $\Delta^{\mathrm{obs}}_T(\theta;\pi)\ge 0$ holds because KL divergence is always nonnegative due to Gibbs' inequality.
\end{proof}

\begin{lemma}[Invertible censoring at $x=B$ preserves KL]\label{lem:invertible-B-preserves-KL}
Assume that for all $t$, $D_t\in[0,B]$ almost surely under the true environment. Fix $h$ and let $q_t^\star(\cdot\mid h),q_{t,\theta}(\cdot\mid h)$ be probability measures on $[0,B]$.
Then
\[
\mathrm{KL}(q_t^\star(\cdot\mid h)\|q_{t,\theta}(\cdot\mid h)) \;=\; \mathrm{KL}\big(P^{\mathrm{obs}}_{h,B}\,\|\,Q^{\mathrm{obs}}_{h,B}\big).
\]
\end{lemma}

\begin{proof}
First, we recall that for any action $x\in[0,B]$, the induced conditional law of the censored observation $O=\psi(x,D)$ is defined by pushforward:
\(
P^{\mathrm{obs}}_{h,x} := q_t^\star(\cdot\mid h) \circ (\psi(x,\cdot))^{-1},
Q^{\mathrm{obs}}_{h,x} := q_{t,\theta}(\cdot\mid h) \circ (\psi(x,\cdot))^{-1}
\). Under the bounded demand assumption, for $x=B$ we have $\psi(B,d)=(d,1)$ for all $d\in[0,B]$.
Define the measurable bijection $\phi:[0,B]\to O_B:= [0,B]\times\{1\}$ by $\phi(d)=(d,1)$, whose measurable inverse is
$\phi^{-1}(s,1)=s$.
By construction, $P^{\mathrm{obs}}_{h,B}=q_t^\star(\cdot\mid h)\circ\phi^{-1}$ and $Q^{\mathrm{obs}}_{h,B}=q_{t,\theta}(\cdot\mid h)\circ\phi^{-1}$.

Let $R := \frac{dq_t^\star(\cdot\mid h)}{dq_{t,\theta}(\cdot\mid h)}$ be the Radon--Nikodym derivative (with the convention $\mathrm{KL}(q_t^\star(\cdot\mid h)\|q_{t,\theta}(\cdot\mid h))=+\infty$ if $q_t^\star(\cdot\mid h)\not\ll q_{t,\theta}(\cdot\mid h)$).
Then $P^{\mathrm{obs}}_{h,B}\ll Q^{\mathrm{obs}}_{h,B}$ and
\(
\frac{dP^{\mathrm{obs}}_{h,B}}{dQ^{\mathrm{obs}}_{h,B}}(o) = R(\phi^{-1}(o))
\)
for $o\in O_B$. Therefore, by change of variables under the pushforward,
\begin{align*}
\mathrm{KL}\big(P^{\mathrm{obs}}_{h,B}\,\|\,Q^{\mathrm{obs}}_{h,B}\big)
&= \int_{O_B} \log\!\Big(\frac{dP^{\mathrm{obs}}_{h,B}}{dQ^{\mathrm{obs}}_{h,B}}(o)\Big)\, dP^{\mathrm{obs}}_{h,B}(o)\\
&= \int_{[0,B]} \log(R(d))\, dq_t^\star(\cdot\mid h)(d)\\
&= \mathrm{KL}(q_t^\star(\cdot\mid h)\|q_{t,\theta}(\cdot\mid h)).  
\end{align*}
Informally, when $X_t = B$, the right-censoring map reveals the demand exactly, so the observed KL mismatch at B equals the latent one-step completion mismatch. The above is a formal proof for this simple intuition.
\end{proof}

\begin{lemma}[Coverage turns observed mismatch into completion mismatch]\label{lem:coverage-obs-to-latent}
Assume for all $t$, $D_t\in[0,B]$ almost surely under the true environment. Moreover, also assume Assumption ~\ref{ass:coverage}. Let $\mu(\cdot\mid h)$ be any action distribution on $[0,B]$ satisfying $\mu(\{B\}\mid h)\ge\eta$.
Then
\begin{equation}\label{eq:latent-by-obs}
\mathrm{KL}(q_t^\star(\cdot\mid h)\|q_{t,\theta}(\cdot\mid h))
\;\le\;
\frac{1}{\eta}\;
\mathbb{E}_{X\sim \mu(\cdot\mid h)}
\Big[\mathrm{KL}\big(P^{\mathrm{obs}}_{h,X}\,\|\,Q^{\mathrm{obs}}_{h,X}\big)\Big].
\end{equation}
\end{lemma}

\begin{proof}
Define the nonnegative function $m(x):=\mathrm{KL}\big(P^{\mathrm{obs}}_{h,x}\,\|\,Q^{\mathrm{obs}}_{h,x}\big)\ge 0$.
By Lemma~\ref{lem:invertible-B-preserves-KL}, $m(B)=\mathrm{KL}(q_t^\star(\cdot\mid h)\|q_{t,\theta}(\cdot\mid h))$.
Therefore, using these two results we get the foll
\begin{align*} 
\mathbb{E}_{X\sim\mu(\cdot\mid h)}[m(X)]
&=
\mathbb{E}[m(X)\mathbf{1}\{X=B\}] + \mathbb{E}[m(X)\mathbf{1}\{X \neq B\}] \\
&\ge
\mathbb{E}[m(X)\mathbf{1}\{X=B\}] \\
&=
m(B)\,\mu(\{B\}\mid h) \\
&\ge
\eta\,\mathrm{KL}(q_t^\star(\cdot\mid h)\|q_{t,\theta}(\cdot\mid h)),
\end{align*}
where the first equality is due to the non-negativity of KL divergence or Gibbs' inequality. Finally, we can rearrange to obtain \eqref{eq:latent-by-obs}, which is the desired result of this Lemma.
\end{proof}

\begin{lemma}[KL decomposition for AR-completion kernels]\label{lem:AR-KL-decomp}
Fix $(t,h)$ and suppose the (true and learned) completion kernels admit autoregressive factorizations
\(
Q^\star_t(d\tilde d_{1:T}\mid h) = q^\star_{t,1}(d\tilde d_1\mid h)\prod_{s=2}^T q^\star_{t,s}(d\tilde d_s\mid h,\tilde d_{1:s-1}),
\)
\(
Q_{t,\theta}(d\tilde d_{1:T}\mid h) = q_{t,1,\theta}(d\tilde d_1\mid h)\prod_{s=2}^T q_{t,s,\theta}(d\tilde d_s\mid h,\tilde d_{1:s-1}),
\)
where the product is the usual iterated kernel product construction.
Define
\begin{equation}\label{eq:def-dh}
d(h) := \mathrm{KL}\big(q^\star_{t,1}(\cdot\mid h)\,\|\,q_{t,1,\theta}(\cdot\mid h)\big),
\end{equation}
and for $s\ge 2$,
\begin{equation}\label{eq:def-delta-sh}
\delta_s(h)
:= \mathbb{E}_{\tilde D_{1:s-1}\sim Q^\star_t(\cdot\mid h)}
\Big[
\mathrm{KL}\big(q^\star_{t,s}(\cdot\mid h,\tilde D_{1:s-1})\,\|\,q_{t,s,\theta}(\cdot\mid h,\tilde D_{1:s-1})\big)
\Big].
\end{equation}
Then
\[
\mathrm{KL}\big(Q^\star_t(\cdot\mid h)\,\|\,Q_{t,\theta}(\cdot\mid h)\big)
=
d(h) + \sum_{s=2}^T \delta_s(h).
\]
\end{lemma}

\begin{proof}
First we apply the two variable chain rule for KL (Lemma~\ref{lem:two-var-chain-rule}) with $(X,Y)=(\tilde D_{1:T-1},\tilde D_T)$. Then we recurse on the first term
$\mathrm{KL}(\mathrm{Law}(\tilde D_{1:T-1})\|\mathrm{Law}_\theta(\tilde D_{1:T-1}))$, splitting off $\tilde D_{T-1}$, etc.
After $T-1$ steps this yields the stated decomposition, with the first marginal term equal to $d(h)$ and the summation of the remaining terms equal to $\delta_s(h)$ by definition.
\end{proof}

\begin{lemma}[Self-contraction implies completion KL is controlled by one-step KL]\label{lem:self-contraction-bound}
Assume Assumption~\ref{ass:self-contraction}. Then for every $(t,h)$,
\[
\mathrm{KL}\big(Q^\star_t(\cdot\mid h)\,\|\,Q_{t,\theta}(\cdot\mid h)\big)
\;\le\;
\bigl(1 + c_{\mathrm{sc}}\log T\bigr)\, d(h).
\]
\end{lemma}

\begin{proof}
Using Lemma~\ref{lem:AR-KL-decomp} followed by Assumption~\ref{ass:self-contraction}, we obtain
\begin{align*}
\mathrm{KL}\big(Q^\star_t(\cdot\mid h)\,\|\,Q_{t,\theta}(\cdot\mid h)\big) 
&= d(h) + \sum_{s=2}^T \delta_s(h) \\
&\le d(h) + \sum_{s=2}^T \frac{c_{\mathrm{sc}}}{s}\, d(h) \\
&= \Bigl(1 + c_{\mathrm{sc}}\sum_{s=2}^T \frac{1}{s}\Bigr) d(h).
\end{align*}

Now, using $\sum_{s=2}^T \frac{1}{s} \le \log T$ to bound the last expression gives the claim of this Lemma.
\end{proof}

\subsubsection{Proof of the Theorem}

\begin{theorem*}[Relation between completion mismatch and censoring-aware predictive objective]\label{thm:deploy-vs-obs}
Assume bounded demand ($D_t\in[0,B]$ a.s.), max-order coverage (Assumption~\ref{ass:coverage}),
and self-contraction (Assumption~\ref{ass:self-contraction}). Then
\[
\Delta^{\mathrm{comp}}_T(\theta;\pi_\theta)
\;\le\;
\frac{1+c_{\mathrm{sc}}\log T}{\eta}\;
\Delta^{\mathrm{obs}}_T(\theta;\pi_\theta).
\]
\end{theorem*}

\begin{proof}[Proof of Theorem \ref{thm:completion_to_ovserved}]
Fix $t$ and condition on $H_{t-1}=h$. The first step is to apply Lemma~\ref{lem:self-contraction-bound} and obtain
\begin{align}
\label{eq:thm3_last_1}
\mathrm{KL}\big(Q^\star_t(\cdot\mid h)\,\|\,Q_{t,\theta}(\cdot\mid h)\big)
\le (1+c_{\mathrm{sc}}\log T)\, d(h).  
\end{align}

The second step is to recognize that $d(h)=\mathrm{KL}(q_t^\star(\cdot\mid h)\|q_{t,\theta}(\cdot\mid h))$, due to the following reasoning. The first-step AR marginal $q^\star_{t,1}(\cdot\mid h)$ is (by construction of the pseudo-generation order)
the conditional law of the current-period latent demand $D_t$ given $h$, i.e. it equals $q_t^\star(\cdot\mid h)$, and similarly
$q_{t,1,\theta}(\cdot\mid h)=q_{t,\theta}(\cdot\mid h)$.

In the third step, letting $\mu(\cdot\mid h)$ be the deployed action distribution at $h$ (i.e. $X_t\mid H_{t-1}=h\sim \mu(\cdot\mid h)$), yields the following using Lemma~\ref{lem:coverage-obs-to-latent}
\begin{align}
\label{eq:thm3_last_2}
d(h)=\mathrm{KL}(q_t^\star(\cdot\mid h)\|q_{t,\theta}(\cdot\mid h))
\le \frac{1}{\eta}\;
\mathbb{E}_{X_t\sim \mu(\cdot\mid h)}
\Big[\mathrm{KL}\big(P^{\mathrm{obs}}_{h,X_t}\,\|\,Q^{\mathrm{obs}}_{h,X_t}\big)\Big].
\end{align}

Combining the \eqref{eq:thm3_last_1} and \eqref{eq:thm3_last_2} above we obtain,
\[
\mathrm{KL}\big(Q^\star_t(\cdot\mid h)\,\|\,Q_{t,\theta}(\cdot\mid h)\big)
\le
\frac{1+c_{\mathrm{sc}}\log T}{\eta}\;
\mathbb{E}_{X_t\sim \mu(\cdot\mid h)}
\Big[\mathrm{KL}\big(P^{\mathrm{obs}}_{h,X_t}\,\|\,Q^{\mathrm{obs}}_{h,X_t}\big)\Big].
\]

Finally, taking expectation over $H_{t-1}$ under $\P^\star_{\pi_\theta}$ and summing over $t=1,\dots,T$
\[
\Delta^{\mathrm{comp}}_T(\theta;\pi_\theta)
\le
\frac{1+c_{\mathrm{sc}}\log T}{\eta}\;
\Delta^{\mathrm{obs}}_T(\theta;\pi_\theta),
\]
as claimed. Moreover, Lemma \ref{lem:delta-obs-excess-Lobs} connects the definition of $\Delta^{\mathrm{obs}}_T(\theta;\pi_\theta)$ in the manuscript (based on excess observed log-loss) to the definition using observed KL mismatch, completing our claim in the Theorem.
\end{proof}

%% file: sec-apx_experiment.tex
\section{Experimental Details}
\label{app:exp-details}

\subsection{ChronosFlow-ICGPS Architecture}
\label{app:exp-implementation}

In this Appendix, we provide some further details regarding the ChronosFlow-ICGPS architecture. Similar to the main paper, we study the architecture under three headings: (a) conditioning context vector layer, (b) conditional normalizing flow (CNF) head, and (c) ICGPS sampler, the first two of which we elaborate below, along with details about training.

\paragraph{Conditioning context vector layer}
We start by describing the prompt structure for the Chronos-2 backbone. Note that Chronos-2 expects a time-series prompt. We encode $H_{t-1}$ using a demand-proxy series and optional auxiliary channels. The demand-proxy series is given by
\(
\widetilde Z_s := S_s\cdot \ind{C_s=1} + X_s\cdot \ind{C_s=0}, s<t,
\)
i.e., revealed demand when uncensored and the known lower bound when censored. We provide Chronos with $(Z_s)_{s<t}$ as the main series, and optionally $(X_s)_{s<t}$ and $(C_s)_{s<t}$ as covariates.

From this prompt, Chronos-2 returns predictive quantiles $\bigl\{q^{\mathrm{Chr}}_{\alpha}(t)\bigr\}_{\alpha\in\cJ}$ on a fixed grid
\(
\cJ := \{0.05,0.10,\dots,0.95\}
\). Secondly, we maintain a KM product-limit estimator treating $C_s=1$ as events at $S_s$ and $C_s=0$ as right-censoring
at $X_s$. Writing the KM survival estimate as $\widehat{S}_{t-1}(\cdot)$, we obtain the KM quantiles as \(
\bigl\{q^{\mathrm{KM}}_{\alpha}(t)\bigr\}_{\alpha\in\cJ},
\)
where $q^{\mathrm{KM}}_{\alpha}(t) := \inf\{z\in[0,B]: 1-\widehat{S}_{t-1}(z)\ge \alpha\}$. Thirdly, the summary statistics are obtained as
\(
\mathrm{stats}(H_{t-1})
:=
\Big(X_{t-1},S_{t-1},\ \widehat{\mu}(\widetilde Z)_{t-1}, \widehat{\mathrm{sd}}(\widetilde Z)_{t-1},\
\widehat{\rho}_{t-1},\ t-1\Big),
\)
where $\widehat{\rho}_{t-1}:=\frac{1}{t-1}\sum_{s<t}\ind{C_s=0}$ is the severity indicator of censoring. The conditioning vector is defined as concatenation of above
\(
\widebar H_t:=\omega(H_{t-1})
:= \Big(\{q^{\mathrm{Chr}}_{\alpha}(t)\}_{\alpha\in\cJ},
        \{q^{\mathrm{KM}}_{\alpha}(t)\}_{\alpha\in\cJ},
        \mathrm{stats}(H_{t-1})\Big)\in\R^{44}
\).
\paragraph{Conditional normalizing flow head.} We model demand via a monotone flow
\(
D_t = T_\theta(\widetilde Z; \widebar h_t),
\)
where $Z\sim\cN(0,1)$ and $T_\theta(\cdot; \widebar h)$ is strictly increasing. We use a piecewise-linear map on latent support
$[-L,L]$ (linear tails beyond) with $K$ bins: an MLP hypernetwork maps $\widebar h_t$ to positive bin widths/slopes via
\texttt{softplus}. The induced CDF and inverse-CDF are given by
\(
F_\theta(d\mid \widebar h_t)=\Phi\!\bigl(T_\theta^{-1}(d;\widebar h_t)\bigr),\) and
\(Q_\theta(u\mid \widebar h_t)=T_\theta\!\bigl(\Phi^{-1}(u);\widebar h_t\bigr)
\), respectively.

\paragraph{Training details.}
We use: AdamW optimizer with learning rate on the order of $10^{-4}$ when fine-tuning the Chronos backbone (when enabled for i.i.d. experiments), AdamW optimizer with learning rate on the order of $10^{-3}$ for CNF training,  minibatch sizes ranging from 128 (Chronos fine-tuning) to 2048 (CNF training), and optional global-norm gradient clipping for stability. We select the best checkpoint by validation censored-NLL.

\subsection{Experiment \ref{subsec:exp4} on real-world dataset}
\label{app:exp4}

This appendix records the dataset-to-episode mapping and evaluation protocol used in
Experiment~\ref{subsec:exp4}. We follow \citet[Sec.~6.3]{hssaine2024censored} for preprocessing,
splits, and the \(\lambda\)-controlled censoring procedure to enable direct comparison.

\paragraph{Dataset construction}
\label{app:exp4-data}

We map a transactional retail dataset, called SuperStore \citep{Sahoo2023SuperstoreSales}, into \emph{episodes} (selling seasons), where each episode is an ordered sequence of periods. We treat each (product, season/store) unit as an episode. Within an episode, records are ordered chronologically to form the sequential history \(H_{t-1}=\{(X_s,O_s)\}_{s\le t-1}\) required by
Algorithm~\ref{alg:nv-gps}. We apply the same filtering and preprocessing choices as
\citet{hssaine2024censored} (e.g., removing incomplete seasons and using their normalization/field
handling) so that the resulting benchmark instances match their setting.
We evaluate across the censoring-control values \(\lambda\) defined in
\citet{hssaine2024censored}; smaller \(\lambda\) induces heavier censoring. For each \(\lambda\), we run
the same season-based evaluation described below.
We meta-train on historical seasons and evaluate online on held-out seasons using the split
convention of \citet{hssaine2024censored}. Online evaluation treats each held-out season as a fresh
episode: the policy selects an order each period and updates only through the evolving censored
history \(H_t\).

\paragraph{Algorithm variants.}
We report two ChronosFlow-ICGPS variants under an identical online wrapper and evaluation protocol:
(i) \textsc{Native}, whose completion kernel is trained only on the target dataset’s historical seasons,
and (ii) \textsc{Meta}, whose completion kernel is pretrained on the other real datasets and then
deployed on the target dataset under the same season-based protocol. The purpose is to isolate
whether cross-dataset pretraining improves robustness.

\paragraph{Results and discussion}

Tables~\ref{tab:exp4-real_1}, \ref{tab:exp4-real_2}--\ref{tab:exp4-real_3} report performance across product categories and
\(\lambda\) values, compared to baselines SAA (sample average approximation), KM (Kaplan-Meier), and RCN \citep{hssaine2024censored}. The observations are summarized below

\begin{itemize}
\item 
\emph{Heavy censoring (\(\lambda\le 3\)): strongest gains.}
Under severe censoring, demand is weakly observed early in the season, making cold-start learning
difficult for methods that rely on effectively uncensored demand estimates. ChronosFlow-ICGPS
(\textsc{Meta}) achieves its largest improvements in this regime; for example, on \textsc{Furniture} at
\(\lambda=1\), \textsc{Meta} achieves \(4.1\) versus RCN \(19.8\) and SAA \(16.0\), and on
\textsc{Technology} at \(\lambda=3\), \textsc{Meta} achieves \(4.2\) (best).

\item
\emph{As censoring relaxes: gaps narrow, but not uniformly.}
As \(\lambda\) increases and more periods effectively reveal demand, performance differences often
shrink. The stabilization point is category-dependent (e.g., some categories tighten earlier than
others), reflecting heterogeneous demand dynamics and noise under the same protocol.

\item 
\emph{Meta vs.\ Native: cross-dataset pretraining improves stability.}
Across \(\lambda\), \textsc{Meta} is typically more stable than \textsc{Native}. A representative
instance is \textsc{Office Supplies} at \(\lambda=15\), where \textsc{Native} degrades to \(18.2\) while
\textsc{Meta} remains at \(12.8\). This suggests cross-dataset pretraining helps
both during heavy censoring and by regularizing adaptation when censoring is light.

\item 
\emph{Conclusion.}
The real-data benchmark supports that completion-based in-context GPS transfers beyond i.i.d.\ synthetic
demand to real seasonal episodes with heterogeneous censoring, with the clearest benefits under
severe censoring and competitive performance elsewhere.
\end{itemize}

\begin{table*}[h]
\centering
\caption{Results of Experiment \ref{subsec:exp4} for the \textsc{Office Supplies} Dataset}
\footnotesize
\setlength{\tabcolsep}{3pt}
\renewcommand{\arraystretch}{1.05}
\begin{adjustbox}{max width=\textwidth,center}
\begin{tabular}{l*{15}{c}}
\toprule
Algorithm & $\lambda$=1 & 2 & 3 & 4 & 5 & 6 & 7 & 8 & 9 & 10 & 11 & 12 & 13 & 14 & 15 \\
\midrule
ChronosFlow-ICGPS (\texttt{Native}) & 46.9 & 26.4 & \textbf{10.4} & 11.7 & 13.5 & 12.1 & \textbf{10.3} & 11.8 & \textbf{10.4} & 12.0 & 11.4 & 18.2 & 18.2 & 18.2 & 18.2 \\
ChronosFlow-ICGPS (\texttt{Meta}) & 48.2 & 34.1 & 12.7 & \textbf{9.8} & \textbf{13.6} & \textbf{10.3} & 12.0 & \textbf{9.9} & \textbf{9.8} & 12.8 & 12.8 & 12.8 & 12.8 & 12.8 & 12.8 \\
SAA & 52.6 & 52.6 & 44.3 & 36.9 & 30.5 & 25.0 & 20.6 & 17.1 & 17.1 & 14.6 & 14.4 & 11.5 & 10.4 & 10.4 & \textbf{9.9} \\
RCN \citep{hssaine2024censored} & \textbf{15.9} & \textbf{15.8} & 15.7 & 15.5 & 15.3 & 15.0 & 14.5 & 14.0 & 13.5 & 12.9 & 12.9 & 11.8 & \textbf{9.8} & \textbf{10.0} & 10.4 \\
Kaplan-Meier & 52.6 & 44.3 & 36.9 & 30.5 & 25.0 & 20.6 & 17.1 & 14.6 & 12.8 & 11.5 & \textbf{10.4} & \textbf{9.9} & \textbf{9.8} & \textbf{10.0} & 10.0 \\
\bottomrule
\end{tabular}
\end{adjustbox}
\label{tab:exp4-real_2}
\end{table*}

\begin{table*}[h]
\centering
\caption{Results of Experiment \ref{subsec:exp4} for the \textsc{Furniture} Dataset}
\footnotesize
\setlength{\tabcolsep}{3pt}
\renewcommand{\arraystretch}{1.05}
\begin{adjustbox}{max width=\textwidth,center}
\begin{tabular}{l*{15}{c}}
\toprule
Algorithm & $\lambda$=1 & 2 & 3 & 4 & 5 & 6 & 7 & 8 & 9 & 10 & 11 & 12 & 13 & 14 & 15 \\
\midrule
ChronosFlow-ICGPS (\texttt{Native}) & 12.6 & 6.9 & \textbf{3.8} & 4.3 & 6.8 & 4.0 & 5.8 & 4.3 & 8.9 & 3.9 & 4.3 & 4.4 & 4.4 & 4.3 & 4.3 \\
ChronosFlow-ICGPS (\texttt{Meta}) & \textbf{4.1} & \textbf{4.2} & 3.9 & 4.1 & 4.3 & 4.2 & 4.0 & 4.0 & 3.9 & 3.9 & 3.9 & 3.9 & 3.9 & 3.9 & \textbf{3.8} \\
SAA & 16.0 & 16.0 & 9.9 & 6.5 & 4.5 & \textbf{3.8} & \textbf{3.8} & \textbf{3.8} & \textbf{3.8} & \textbf{3.8} & \textbf{3.8} & \textbf{3.8} & \textbf{3.8} & \textbf{3.8} & \textbf{3.8} \\
RCN \citep{hssaine2024censored} & 19.8 & 19.0 & 17.4 & 15.8 & 11.4 & 4.0 & \textbf{3.8} & \textbf{3.8} & \textbf{3.8} & \textbf{3.8} & \textbf{3.8} & \textbf{3.8} & \textbf{3.8} & \textbf{3.8} & \textbf{3.8} \\
Kaplan-Meier & 16.0 & 9.9 & 6.5 & \textbf{4.5} & \textbf{3.8} & \textbf{3.8} & \textbf{3.8} & \textbf{3.8} & \textbf{3.8} & \textbf{3.8} & \textbf{3.8} & \textbf{3.8} & \textbf{3.8} & \textbf{3.8} & \textbf{3.8} \\
\bottomrule
\end{tabular}
\end{adjustbox}
\label{tab:exp4-real_3}
\end{table*}